\def\year{2020}\relax
\documentclass[letterpaper]{article} % DO NOT CHANGE THIS
\usepackage{aaai20}  % DO NOT CHANGE THIS
\usepackage{times}  % DO NOT CHANGE THIS
\usepackage{helvet} % DO NOT CHANGE THIS
\usepackage{courier}  % DO NOT CHANGE THIS
\usepackage[hyphens]{url}  % DO NOT CHANGE THIS
\usepackage{graphicx} % DO NOT CHANGE THIS
\usepackage{graphicx}  % DO NOT CHANGE THIS
\usepackage{amsmath}
\usepackage{soul}
\usepackage{booktabs}
\usepackage{amssymb}
\usepackage[utf8]{inputenc}
\usepackage{mathrsfs}
\usepackage{amsthm}
\usepackage{textcomp}
\usepackage{xspace}
\usepackage[misc]{ifsym}
\usepackage{float}
\usepackage{subfigure}
\graphicspath{{Figures/}}
\usepackage[linesnumbered,vlined,ruled]{algorithm2e}
\usepackage{algpseudocode}
\SetKwInOut{Input}{input}
\SetKwInOut{Output}{output}

\newtheorem{definition}{Definition}
\newtheorem{theorem}{Theorem}
\newtheorem{proposition}{Proposition}

\newtheorem{examp}{Example}

\newcommand{\Actions}{\mathcal{A}}

\newcommand{\Operators}{\mathcal{O}}
\newcommand{\Cactions}{\mathcal{C}}

\newcommand{\HtnDomain}{\mathfrak{D}}
\newcommand{\Methods}{\mathcal{M}}
\newcommand{\Method}{\Methods}
\newcommand{\HtnProblem}{\mathcal{P}}

\newcommand{\Lvar}{\mathcal{L}}

\renewcommand{\P}{\mathcal{P}}

\newcommand{\Constraint}{\prec}
\newcommand{\Constraints}{\Constraint}

\newcommand{\tn}{\mathsf{tn}}
\newcommand{\pre}{\mathsf{pre}}
\newcommand{\add}{\mathsf{add}}
\newcommand{\del}{\mathsf{del}}

\newcommand{\plan}{\sigma}
\newcommand{\name}{\mathsf{name}}

\newcommand{\Dtree}{\mathcal{T}}

\newcommand{\Tasksof}{T}
\newcommand{\yield}{\vartheta}

\newcommand{\innernodes}{N_\Dtree}
\newcommand{\insertTask}{t'}
\newcommand{\primTask}{t_p}

\newcommand{\aspplus}{\textsf{+}}
\newcommand{\aspprec}{\ll}

\newcommand{\insertActions}{I}

\newcommand{\improveFunction}{\rho}

\newcommand{\prefer}{\leq_P}
\newcommand{\strictPrefer}{<_P}

\newcommand{\complete}{\tau}
\newcommand{\inverseComplete}{\complete}

\newcommand{\SubTCandidates}{\Delta_t}

\newcommand{\Instances}{\mathcal{I}}
\newcommand{\closure}{\mathsf{closure}}

\usepackage{color}            % added by hankz, Nov 22, 2019
\allowdisplaybreaks

\makeatletter
\def\blfootnote{\xdef\@thefnmark{}\@footnotetext}
\makeatother
\title{Refining HTN Methods via Task Insertion with Preferences (ArXiv Version)}
\author{Zhanhao Xiao,\textsuperscript{\rm 1} Hai Wan,\textsuperscript{\rm 1}\thanks{Corresponding author} Hankui Hankz Zhuo,\textsuperscript{\rm 1}\thanks{Key Laboratory of Machine Intelligence and Advanced Computing, Ministry of Education, China; Guangdong Province Key Laboratory of Big Data Analysis and Processing, China.} Andreas Herzig,\textsuperscript{\rm 2,3} Laurent Perrussel,\textsuperscript{\rm 2,3} Peilin Chen\textsuperscript{\rm 1}\\
\textsuperscript{\rm 1}School of Data and Computer Science, Sun Yat-sen University, Guangzhou, China\\
\textsuperscript{\rm 2}IRIT, Toulouse, France\\
\textsuperscript{\rm 3}University of Toulouse, Toulouse, France\\
}
 
\begin{document}

\maketitle
\looseness =-1
\begin{abstract}
Hierarchical Task Network (HTN) planning is showing its power in real-world planning. Although domain experts have partial hierarchical domain knowledge, it is time-consuming to specify all HTN methods, leaving them incomplete. On the other hand, traditional HTN learning approaches focus only on declarative goals, omitting the hierarchical domain knowledge.
In this paper, we propose a novel learning framework to refine HTN methods via task insertion with completely preserving the original methods.
As it is difficult to identify incomplete methods without designating declarative goals for compound tasks, we introduce the notion of prioritized preference to capture the incompleteness possibility of methods.
Specifically, the framework first computes the preferred completion profile \emph{w.r.t.} the prioritized preference to refine the incomplete methods.
Then it finds the minimal set of refined methods via a method substitution operation.
Experimental analysis demonstrates that our approach is effective, especially in solving new HTN planning instances.
%
%Hierarchical task network (HTN) planning technique is used in a growing number of real-world applications. However in many domains, as there exist thousands of cases, it is difficult and time-consuming for domain experts to specify all HTN methods to cover all desirable plans.
%Normally, domain experts have partial hierarchical domain knowledge to define HTN methods which are likely to be incomplete.
%However, the traditional HTN-method learning approaches only focus on the declarative goals
%and omit these hierarchical domain knowledge.
%In this paper, we focus on refining the incomplete methods and propose a novel framework to learn HTN methods by refining methods via inserting subtasks.
%To capture the possibility of incomplete methods, we introduce the notion of prioritized preference on methods and propose an approach to computing the preferred refined method set with respect to a given preference.
%By taking experiments on three well-known planning domains, we demonstrate that our approach is effective, especially on solving new HTN planning instances.
\end{abstract}

\section{Introduction}
Hierarchical task network (HTN) planning techniques \cite{erol1994htn} are increasingly used in a number of real-world applications \cite{lin2008web,DBLP:journals/aicom/BehnkeSKBSDDMGB19}.
In the real-world logistics domain, such as Amazon and DHL Global Logistics, the shipment of packages is arranged via decomposition into a more detailed shipment arrangement in a top-down way according to the predefined HTN methods.
In practice, there exist a vast number of cases occurring, such as the delay caused by the weather, leading that it is difficult and time-consuming for humans to find all complete methods for all actions.
This suggests that it is important to learn complete methods. %to help humans to improve the HTN domain.

Normally, domain experts have partially hierarchical domain knowledge, which possibly is not sufficient to cover all desirable solutions \cite{kambhampati1998hybrid}.
%Different from classical planning which pursues an executable plan to achieve the declarative goal, the solution to the HTN planning problem requires to consider the hierarchical procedural goals, which are given by HTN methods.
On one hand, with partially hierarchical domain knowledge, a planner may fail to obtain a solution via decomposition according to the given methods.
The main reason lies in that the given method set is incomplete, which includes at least an incomplete method lacking subtasks.
%Keeping the hierarchical procedural knowledge,
On the other hand, the hierarchical domain knowledge comes from the experience and investigation of domain experts, which normally are supposed to be necessary.
However, the traditional approaches to learning HTN methods, such as \cite{DBLP:conf/aaai/HoggMK08}, only concentrate on declarative goals of compound tasks and omit hierarchically procedural knowledge obtained from domain experts.%, which cannot be replaced simply by declarative goals.
Indeed, these procedural knowledge cannot be expressed by only operator structures (action models),
which can be concluded from \cite{Goldman09,holler2014language}.
%, the procedural goals cannot be expressed by only primitive action models.
%For example, every package needs a security check before being uploaded into the plane. If the action model is not complete, such as the ‘check’ action has not the effect ‘checked’, the declarative goal may not capture it.
%Actually, as shown in
Therefore, in this paper, we focus on refining HTN methods and keeping the original hierarchical domain knowledge from domain experts.

\begin{figure*}[!htp]
\centering
\includegraphics[width=0.8\linewidth]{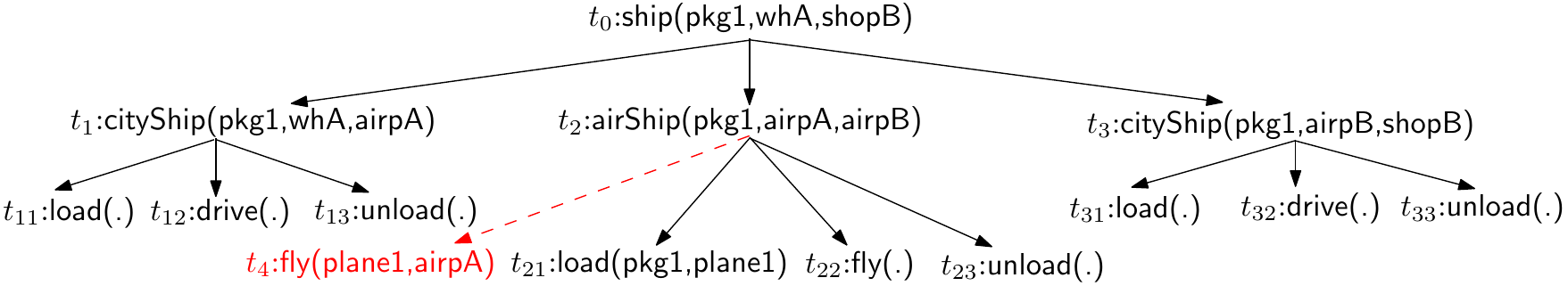}
%%%-3mm}
\caption{An example of a decomposition tree from incomplete methods (the parameters of some actions are hidden). The initial task $\mathsf{ship(pkg1,whA,shopB)}$ is decomposed into a sequence of primitive tasks (the black leaves) according to the original methods. But when plane1 is not in airport A, the sequence is not executable. It becomes executable if arranging plane 1 to airport A before loading the package, which implies that $\mathsf{fly(plane1,airpA)}$ should be considered as a subtask of $\mathsf{airShip}$. }\label{fig:example_of_incomplete_method}
%%%-5mm}
\end{figure*}

To tackle the method incompleteness,
Geier and Bercher (\citeyear{geier2011decidability}) proposed a hybrid planning formalization, \emph{HTN planning with task insertion (TIHTN planning)}, to allow generating plans via decomposing tasks according to the methods but also inserting tasks from outside the given methods.
Example \ref{examp:incomplte_methods} shows a plan with task insertion (called TIHTN plan) for an HTN problem with incomplete methods.
Actually, the TIHTN plan offers a reference to refine the methods by adding the inserted tasks.
For example, the method of $\mathsf{airShip}$ can be refined by adding $\mathsf{fly}$ as its subtask.
It allows us to absolutely preserve the original domain knowledge to refine the incomplete methods, further to assist domain experts to improve the HTN domain.
%By refining methods, we obtain new methods to compute plans which includes the missing tasks.
%, resulting in a new decomposition tree.

\begin{examp}\label{examp:incomplte_methods}
Consider an example in the logistics domain, suppose every task has only one method and a decomposition tree is shown in Figure 1.
The initial task $\mathsf{ship(pkg1,whA,shopB)}$ is to ship a package from city A to city B and
it has a method: to ship the package from the warehouse to the airport by truck, from city A to city B by plane and from the airport to the shop by truck.
But in case that the plane is not in the airport of city A, the air transportation task $\mathsf{airShip(pkg1,airpA,airpB)}$ %to ship the package from city A to city B by plane
cannot be accomplished, neither can the initial task.
%the package initially is in the post office of city A and it needs a subtask to ship it from the post office to the airport in city A  by truck which the HTN designer has missed.
%When the HTN designer writes the hierarchical domain
When arranging the plane to airport A, $\mathsf{fly(plane1,airpA)}$, is done before loading to the plane, an executable plan is found.
%If it is not allowed to insert actions, there is no plan to accomplish the task $\mathsf{ship(pkg1,whA,shopB)}$.
%in consequence the missing subtask is not in the task network and
%will never be produced via decomposition.
%in consequence the initial logistics task cannot be accomplished.
%
%The approach of TIHTN planning indeed fills up the plans by adding the missing actions to accomplish the initial tasks.
%So a TIHTN planner can provide clues for the designer to spot the mistakes and to complete the HTN methods.
\end{examp}

%Whereas, it leads to another problem: where the tasks should be inserted to refine methods?
%It is however challenging,
Whereas, even if a TIHTN plan is found, without designating declarative goals for compound tasks,
it is still difficult to identify incomplete methods -- because an inserted task can be considered as a missed subtask of different methods.
%-- there are a vast number of refinement ways to be considered.
%as tasks can be inserted in various methods and an exponential number of method sets need to be considered.
For example, the inserted task $\mathsf{fly}$ can also be used to refine the method of $\mathsf{cityShip}$.
In practice, the missing of subtasks happens more likely on certain methods than on some other methods.
%For example, the task $\mathsf{ship}$ is decomposed into the inter-city shipment and the intra-city shipment, while decomposing the task $\mathsf{airShip}$ varies with the place where the plane stays.
It motivates us to introduce  the notion of prioritized preference on methods to capture the incompleteness possibility of methods. %, which comes from the experience of domain experts.
%Observe that the more detailed tasks are more sensitive to situations and more easily influenced.
%It leads to a priority on the methods: some methods have a high priority to be refined.
%%An excess of learned methods will slow down problem-solving, so we hope to learn as few refined methods as possible.
% % into account.

%\looseness =-1

%The contributions of this paper are listed as follows.
Our contributions are listed as follows.
First, we introduce the notion of completion profiles to refine HTN methods from TIHTN plans.
Second, we propose a framework $\textsc{MethodRefine}$ to refine HTN methods with completely preserving the original methods.
Specifically,
we first compute the preferred completion profile \emph{w.r.t.} the prioritized preference to refine incomplete methods and then propose a method substitution operation to obtain the preferred set of refined methods.
%We also propose an approach to computing the preferred refined method set with respect to a given preference.
Third, by taking experiments on three well-known planning domains, we compare our approach with different method incompleteness against the classical HTN learning approach, HTN-MAKER, on the ability of solving new instances in the same domain and show that our approach is more effective. %, especially in solving new HTN planning instances.
%, especially on solving new HTN problems.
%We also compare the performances of our approach with different prioritized preferences and show that the preference which
%There are two reasons why we should learn methods from the decomposition tree instead of the sequential plan:
%1) theoretically, the HTN planning naturally contains hierarchical and procedural knowledge while the classical planning cannot represent them;
%2) practically, the methods defined by the domain designer should be sound but not always complete, so the decomposition tree should be sound.

%The rest of the paper is structured as follows.
%We first review previous related work.
%Then we give an overview key notions in TIHTN planning and a progression-based algorithm. %to briefly explain ASP programs..
%Next we show how to complete HTN methods and present $\textsc{MethodRefine}$ to learn methods.
%Finally we evaluate $\textsc{MethodRefine}$ and conclude the paper.

%%%-2mm}
\section{Related Work}
Besides HTN-MAKER, there are a number of HTN learning approaches \cite{DBLP:conf/ecai/LotinacJ16,DBLP:journals/ai/ZhuoM014}
which only focus on declarative goals and omit hierarchically procedural knowledge.
Specially, they require annotated preconditions and effects of compound tasks and only consider the declarative goals like classical planning, so they require a complete executable plan as input. Whereas, %it is not a simple task to obtain complete plans, particularly when it involves thousands of situations. Furthermore,
in many domains, it is difficult to verify the correctness of the task annotations. %when taken as input.
%Another {DBLP:conf/ecai/LotinacJ16}
%In this paper, we propose a novel framework to learn HTN methods from HTN instances with an incomplete method set, which always cannot generate executable plans only via decomposition.
Besides, these approaches restrict the tasks in the methods to be totally ordered, while we allow them to be partially ordered.
Last but not least, comparing with these approaches, we also consider the prioritized preference on the learned methods. %while they do not.

%Besides those we mentioned above, there have been action model learning approaches related with our work.
Another related work is \cite{DBLP:conf/kcap/GarlandRR01} which proposes an approach to construct and maintain hierarchical task models from a set of annotated examples provided by domain experts.
Similar to the annotated tasks, obtaining these annotated examples is difficult and needs a lot of human effort.
Our work also is related to the works on learning the precondition of HTN methods \cite{DBLP:conf/icml/IlghamiMNA05,DBLP:conf/aaai/XuM05}, which takes the hierarchical relationships between tasks, the action models, and a complete description of the intermediate states as input. The similar work also includes \cite{DBLP:conf/icml/NejatiLK06} and \cite{ReddyT97}, which used means-end analysis to learn structures and preconditions of the input plans. The precondition and effect of primitive actions can also be learned in \cite{DBLP:conf/ijcai/ZhuoHHYM09}.
All these approaches of learning method precondition require a complete method set as input.

%The work on hybrid planning which combines classical planning and HTN planning is also related with our work.
%By relaxing the restriction of generating plans only via decomposition,
%TIHTN planning allows to generate
For TIHTN planning,
Geier and Bercher (\citeyear{geier2011decidability}) first addressed the decidability of propositional TIHTN planning.
Later Alford \emph{et al.} (\citeyear{alford2015tighttihtn}) proved that propositional TIHTN planning is EXPTIME-complete and proposed an acyclically searching approach to compute TIHTN plans, which provides a route to refine incomplete HTN methods.
%In this paper, we aim to refine HTN methods with the help of TIHTN planning.

%%%-3mm}
\section{Definitions}
%%\begin{definition}[Task Network]
%In this section we recall the notions of propositional TIHTN planning \cite{geier2011decidability} and propose a progression-based algorithm to find TIHTN plans.
%%adjust the presentation of TIHTN planning \cite{geier2011decidability} by using a pair of an action and an integer to denote a task.
%TIHTN planning %is originally proposed by \cite{geier2011decidability} which is in term of the propositional language and later
%is generalized into the first-order language with finite relations and constants, called lifted TIHTN planning \cite{alford2015tighttihtn}.
%As solutions of lifted TIHTN planning are defined in a grounding way, here we only consider the propositional language\footnote{It is not difficult to generalize the formalization into first-order.}.
%%We first define a function-free first order language $\Lvar$ from a finite set of relations and constants and a set of variables.
%%Then we define a finite-state transition system in which every state $s$ is a finite set of ground atoms of $\Lvar$.

We adapt the definitions of lifted HTN planning \cite{alford2015tighttihtn}.
%When lifted HTN planning is defined in a grounding way, in this paper we only give the definitions of propositional HTN planning.
First, we define a function-free first order language $\Lvar$ from a set of variables and a finite set $\Lvar_0$ of predicates and constants.
We take parts of variables in $\Lvar$ as \emph{task symbols} to identify \emph{tasks}.
A state is any subset of ground atoms in $\Lvar$ and the finite set of states are denoted by $2^\Lvar$.
In HTN planning, actions (or task names), noted $\Actions$, are syntactically first-order atoms,
which are classified into two categories: the actions the agent can execute directly are called \emph{primitive actions} or \emph{operators}, noted $\Operators$, while the rest are called \emph{compound actions}, noted $\Cactions$.
Every primitive action $o$ is a tuple $(\name(o),\pre(o),\add(o),\del(o))$ where %$\name(o)$ is first-order atom called the name\footnote{We always use
$\name(o)$ consists of a predicate out of $\Lvar$ and a list of variables, called its name;
$\pre(o)$ is a first-order logic formula, called its precondition; $\add(o)$ and $\del(o)$ are two conflict-free sets of atoms called its positive and negative effect.
%We use $\Operators$ to denote the set of operators and $\Actions$ to denote all actions, so $\Actions {=} \Operators \cup \Cactions$.
 %,
Then we define a state-transition function $\gamma: 2^{\Lvar} \times \Operators \longrightarrow 2^{\Lvar}$:
$\gamma(s,o)$ is defined if $o$ is \emph{applicable} in $s$ (\emph{i.e.}, $s \models \pre(o)$);
$\gamma(s,o){=}(s {\setminus} \del(o)) {\cup} \add(o)$.
%A primitive action $o$ is \emph{applicable} in a state $s$ if $s \models \pre(o)$, which results in a state $\gamma(s,o){=}(s {\setminus} \del(o)) {\cup} \add(o)$.
%We define a function-free first order language $\Lvar$ from a finite set of relations and constants $\L_0$ and a set of variables.
%In HTN planning, task names, noted $\mathcal{N}$, represent acts to accomplish and are syntactically first-order atoms in $\Lvar$.
%%Task names typically contain variables that can be eliminated via grounding.
%%For example, by grounding the task name ``\textsf{openDoor(X)}'' where X is quantified as a door, we can obtain ``\textsf{openDoor(a)}'', ``\textsf{openDoor(b)}'', etc. which in fact are a set of task names.
%Those acts in $\mathcal{N}$ which can be executed directly by the agent are called $operators$, noted $\Operators$, while the rest are called $compound$ task names.
%Each operator $o \in \Operators$ is represented as a triple of formulas: precondition $\pre(o)$, positive effect $\add(o)$, and negative effect $\del(o)$ %via functions $\pre,\add,\del: \O \rightarrow \Lvar$,
%where $\add(o),\del(o)$ are sets of atomic formulas.
%
%With a set of operators $\Operators$, we can define a state-transition function $\gamma: 2^{\L_0} \times \O \longrightarrow 2^{\L_0}$ for the problem, where:
%\begin{itemize}
%  \item A state is a subset of the ground atoms in $\L_0$
%%  \item $o$ is applicable in a state $s$ iff $s \models \pre(o)$
%  \item $\gamma(s,o)$ is defined iff $s \models \pre(o)$
%  \item $\gamma(s,o) = (s \setminus \del(o)) \cup \add(o)$ if $\gamma(s,o)$ is defined
%\end{itemize}
A sequence of primitive actions $o_1{,}{...}{,}o_n $ is \emph{executable} in a state $s_0$ iff there is a state sequence $s_1{,}{...}{,}s_n$ such that $\forall{1 \leq i \leq n}, \gamma(s_{i-1},o_i) {=} s_i$. %and $o_i$ is applicable in $s_{i-1}$.

%Remember that
%For a relation $R$, we use $\closure(R)$ to denote its transitive closure, \emph{i.e.} the smallest transitive relation containing $R$.
Given a set $R$, we use $\vec{R}$ to denote the set of all sequences over $R$ and use $|R|$ to denote the cardinality of $R$.
For its subset $X$ and a function $f: R \longrightarrow S$, its restriction to $X$ is
$f |_X = \{(r,s) \in f \mid r \in X\}$.
For a binary relation $Q \subseteq R \times R$, we define
its restriction to $X$ by $Q |_X = Q \cap (X \times X)$.

%The restriction of a relation $Q \!\subseteq\! R \!\times\! R$
%to a set $X$ is $Q |_{\!X} \!=\! Q \cap (X \!\times\! X)$ and the restriction of a function $f:\! R {\longrightarrow} S$ to $X$ is $f |_{\!X} {=} \{(r,s) {\in} f \mid r {\in} X\} {=} f {\cap} (X {\times} S)$.
%We %extend %union operator $\cup$ into tuples: $(Q_1,Q_2) \cup (Q'_1,Q'_2) = (Q_1 \cup Q_1',Q_2 \cup Q_2')$
%extend function into sequence: $f(\langle t_1,\ldots,t_n \rangle) = \langle f(t_1),\ldots,f(t_n) \rangle$.

\smallskip
\noindent\textbf{Task networks.}
A task network  %over $\mathcal{N}$
is a tuple $\tn {=} (T,\Constraint,\alpha)$ where
$T$ is a set of tasks, $\Constraint\! {\subseteq}  T{\times} T$ is a non-empty set of ordering constraints over $T$ and $\alpha\!:\!T\!\! \longrightarrow\!\! \Actions$ labels every task with an action.
%\end{itemize}
%\end{definition}
%In a task network, each task is identified by a task symbol in $T$.
%The function $\alpha$ associates every task symbol to an action.
Every task is associated to an action and the ordering constraints restrict the execution order of tasks.
%Thus, in this paper without ambiguity we use ``task'' to indicate ``task symbol''.
%In a task network, every task is a pair of an action $a$ and an integer identifier $i$. %which is a positive .
%With the same action, different identifiers indicate different tasks.
%For a task $t=(a,i)$, we define $\alpha(t)=a$. %and $id(t)=i$.
A task $t$ is called \emph{primitive} if $\alpha(t)$ is primitive (otherwise called \emph{compound}), and called \emph{ground} if $\alpha(t)$ is ground.
A task network is called primitive if it contains only primitive tasks, and called ground if it contains only ground tasks.

\smallskip
\noindent\textbf{HTN methods.}
Compound actions cannot be directly executed and need to be decomposed into a task network according to HTN methods.
An \emph{HTN method} $m{=}(c,{\tn_m})$ consists of %$\psi$ is a formula, called the method's precondition,
a compound action $c$ (called \emph{head}) and a task network ${\tn_m}$ whose inner tasks are called \emph{subtasks}.
Generally, an HTN method includes variables, which can be grounded as actions are grounded.
Note that a compound action $c$ may have more than one HTN method. %, denoted by $m^c_1,...,m^c_k,...$.
%The method $(c,\tn_m)$ means that %if the precondition $\psi$ holds, then
%action $c$ can be reduced by subtask network $\tn_m$.

Intuitively, decomposition is done by selecting a compound task, adding its subtask network and replacing it.
%To distinguish difference occurrences of an action, the identifiers of the new introduced tasks are required to be fresh.
%Then, to task network $\tn_1$ into $\tn$, it means to add an isomorphic task network $\tn^i_1$ with $\tn_1$.
The constraints about the decomposed task $t$ are propagated to its subtasks: the tasks before $t$ are before all its subtasks and the tasks after $t$ are after all its subtasks.
%
%Formally, for a task network $\tn=(T,\Constraint,\alpha)$, a compound task $t \in T$ and its method $m=(c,\tn_m)$, suppose a task network $\tn'_m=(T'_m,\Constraint'_m,\alpha'_m)$ where $\tn'_m \cong \tn_m$ and $T'_m \cap T = \emptyset$.
%%We define a decomposition step as a pair of a compound task $(c,i)$ and its HTN method number $k$.
%Then the decomposition of $t$ in $\tn$ by $m$ into a task network $tn'$, %written by $\tn \xrightarrow[t,m]{} \tn'$,
%%%%-0.5mm}
%where $\tn' = (T',\Constraint',\alpha')$
%is given by:
%
%%%-6mm}
%\begin{align*}
%T' :=\ & (T \setminus \{t\}) \cup T'_m \\
%\Constraint' \ :=\ & {\Constraint} |_{T'} \cup \Constraint_m  \cup \{(t_1,t_m) \mid (t_1,t) \in \Constraint, t_m \in T'_m\}\\
%     & \phantom{ {\Constraint} |_{T'} \cup \Constraint_m  } \cup \{(t_m,t_2) \mid (t,t_2) \in \Constraint, t_m \in T'_m\}\\
%\alpha' :=\ & {\alpha} |_{T'} \cup \alpha_m
%\end{align*}
%%%-6mm}
%
%If $\tn'$ is reachable by any finite sequence of decompositions of $\tn$, we write $\tn \longrightarrow_{D}^* \tn'$.

%Two task networks are called \emph{isomorphic} if they only

A task network $\tn \!= \!(T,\Constraint,\alpha)$ is a grounding of another task network $\tn'\!= \!(T',\Constraint',$ $\alpha')$ if
%Two task networks $\tn \!= \!(T,\Constraint,\alpha)$ and $\tn' \!= \!(T',\Constraint',$ $\alpha')$ are called \emph{isomorphic}, noted $\tn \!\cong\! \tn'$, if and
there exists a bijection $f\!:\! T \!\longrightarrow\! T'$ such that
$\alpha(t)$ is a grounding of $\alpha(f(t))$ and for all $t_1 \Constraint t_2$, $f(t_1){\Constraint'} f(t_2)$.

\smallskip
\noindent\textbf{HTN problems.}
An HTN planning \emph{domain} is a tuple $\HtnDomain=(\Lvar,\Operators,\Cactions,\Methods)$ where $\Methods$ is a set of HTN methods and $\Operators \cap \Cactions =\emptyset$.
We call a pair $(s_0,t_0)$ an \emph{instance} where $s_0$ is a ground initial state and $t_0$ is a ground initial task.
An HTN \emph{problem} is a tuple $\P= (\HtnDomain,s_0,t_0)$. %where there exists at least a method in $\Methods$ to decompose $t_0$.

\smallskip
\noindent\textbf{Solutions.}
A solution to an HTN problem $\HtnProblem = (\HtnDomain,s_0,t_0)$ is a valid decomposition tree $\Dtree$ \emph{w.r.t.}\ $\HtnProblem$ and
we say $(s_0,t_0)$ is solved under $\HtnDomain$ and is satisfied by $\Dtree$.

In different literature, the solution to the HTN problem has different forms: mostly a plan (such as \cite{erol1994htn}), a primitive task network (such as \cite{Behnke2017}) and a list of decomposition trees (such as \cite{DBLP:journals/ai/ZhuoM014}).
In this paper, we consider a solution to the HTN problem as a decomposition tree rooted in the initial task $t_0$.
%which actually combines the decomposition trees induced by the compound tasks in the initial task network via a unique root.
%%As the initial task network may contain more than one task and every task generates a tree via decompositions, it leads to a forest for the initial task network.
%%In order to integrate them into a tree, we create a ``root'':
%Then we introduce a fresh compound action $\ctop$ not occurring in $\HtnDomain$ with meaning ``to accomplish all the initial tasks'' and use a new task $r$ to identify it.
%%To capture an HTN problem, we will restrict $\ctop$ to only decompose into the original initial task network $\tn_I$ for the ground HTN problem.
%Then we define decomposition tree for an HTN domain.

%\begin{definition}[\textbf{Valid decomposition trees}]\label{def_decomposition_tree}
%For an HTN problem $\HtnProblem=(\Lvar,\Cname,\Operators,$ $\Method,s_0,t_0)$, one
A decomposition tree (DT) is a tuple $\Dtree = (T,E,\Constraints,\alpha,\beta)$
where
%\begin{itemize}
  %\item
  $(T,E)$ is a tree, with nodes $T$ and with directed edges $E: T \longrightarrow \vec{T}$ mapping each node to an ordered list of its children;
  %\item
  $\Constraints$ is a set of constraints over $T$;
  %\item
  function $\alpha: T \longrightarrow \Actions$ links tasks and actions; %where $\alpha(r)=\ctop$; %such that $\ctop$ is a new compound action; %which is decomposed into a grounding of the initial task network. %with an intuition ``to accomplish all initial tasks''
  %\item
  function $\beta: T \longrightarrow \Method$ labels every inner node with an HTN method. %where $\Method' = \Method \cup \{ (\ctop,\tn_0)\}$ and $\tn_0$ is some task network;
%\end{itemize}
%\end{definition}

%Then we define the subtree of $\Dtree$ induced by a node $t$, as $\Dtree[t] = (T',E',{\Constraints}|_{T'},$ $\alpha |_{T'},\beta|_{T'})$ where $(T',E')$ is the subtree in $(T,E)$ which is rooted in $t$.
%Note that if $t$ is a leaf node, $\Dtree[t]$ is itself.

%For two nodes $t_1,t_2 \in T$, we say $t_1$ is a precedent of $t_2$, written as $t_1 \aspprec t_2$, iff for their lowest common ancestor $t_3$, the child of $t_3$ which is also the ancestor of $t_1$ is ahead of the respective ancestor of $t_2$ in $E(t_3)$.
We use $\aspprec$ to denote the transitive closure of $\Constraints$ and the order defined by $E$. We say $t_1$ is a predecessor of $t_2$ if $t_1 \aspprec t_2$.
Dually, we also say $t_2$ is a successor of $t_1$.
%For a partial order $\aspprec$, we use $\BEF_{\aspprec}(t)$ and $\AFT_{\aspprec}(t)$ to denote the predecessors and successors of $t$.
According to $\aspprec$, we say the sequence constituted by the ground leaf nodes of $\Dtree$ is its plan, denoted by $\yield(\Dtree)$.
%Note that $\Constraints$ is a partial order but it
%For short, we use $\yield$ to denote $\yield(r)$ where $r$ is the root.
%We say the start of a task $t$ is the first leaf node in $\yield$, denoted by $\firstsubtask$ and the end of $t$ is the last leaf node in $\yield$, denoted by $\lastsubtask$.
%Note that if $t$ is a leaf node, its start $\firstsubtask$ and end $\lastsubtask$ are itself.

\begin{definition}[\textbf{Valid DTs}]\label{def_valid_decomp_tree}
A DT $\Dtree$ is valid \emph{w.r.t.}\ an HTN problem %$\groundHtnProblem= (\L,\Cname \cup \{\ctop\},\Operators,\groundHtnDyntheory,\Method \cup \Method_{top},s_I,\tn'_I)$
$\HtnProblem = (\HtnDomain,s_0,t_0)$
iff its plan $\yield(\Dtree)$ is executable in $s_0$ and
its root is $t_0$ % where $\beta(r)=(\ctop,\tn_I)$, %$\ch(\Dtree,t_0)$
and for every inner node $t$ where $\beta(t)=(c,\tn_m)$, it satisfies:
\begin{enumerate}
  \item $\alpha(t)=c$;
  \item %there exists $\sub(t) \subseteq E(t)$ such that $(\sub(t),\Constraints|_{\sub(t)},\alpha |_{\sub(t)}) \cong \tn_m$;
  $(E(t),\Constraints\!|_{E(t)},\alpha |_{E(t)})$ is a grounding of $\tn_m$;
  %\item for every $t' \in \sub(t)$: $(t',\top,\taskend),(\taskstart,\top,t') \in \Constraints$
% \item \label{def_valid_decomp_tree_criterion_constraints} for every $t' \in \Tasksof(\Dtree) \cup \{\nil\}$:
  %\begin{itemize}
  \item if $(t, t')\in \Constraints$ then for every $st \in E(t)$, $(st,t') \in \Constraints$; %where $st$ is the last task in $E(t)$;
  \item if $(t', t)\in \Constraints$ then for every $st \in E(t)$, $(t', st) \in \Constraints$; %where $st$ is the first task in $E(t)$;
  %\end{itemize}
  \item there are no $t_1,t_2$ such that $t_1 \aspprec t_2$ and $t_2 \aspprec t_1$.
  %there is no constraint in $\Constraints$ except for those demanded by the above criteria. % 3. and 4.
\end{enumerate}
\end{definition}

\begin{examp}[Example \ref{examp:incomplte_methods} cont.]\label{examp:solution}
If \textsf{plane1} is already at airport A in $s_0$, the DT drawn with \textbf{black} arrows in Figure \ref{fig:example_of_incomplete_method} is %\textsf{load};\textsf{drive};\textsf{unload};\textsf{load};\textsf{fly};\textsf{unload};\textsf{load};\textsf{drive};\textsf{unload}\footnote{Here we omit the arguments and identifiers.}\!
%which is executable in the initial state.
%So, it is
a solution to the HTN problem with a plan
$\plan_1=\langle${\small\textsf{load};\textsf{drive};\textsf{unload};\textsf{load};\textsf{fly};\textsf{unload};\textsf{load};\textsf{drive};\textsf{unload}}$\rangle$.
%is its plan.
%is an HTN plan; %otherwise, there is no HTN plan but
%the sequence \textsf{load};\textsf{drive};\textsf{unload};\textbf{fly};\textsf{load};
%\textsf{fly};\textsf{unload};\textsf{load};\textsf{drive};
%\textsf{unload}
%is a TIHTN plan.
\end{examp}

\section{Refining Methods via Task Insertion}
In this paper, we focus on the HTN problem with an incomplete method set, where there is no valid decomposition tree \emph{w.r.t.}\ the problem.
In other words, there is no executable plan obtained only by applying methods.
By allowing inserting tasks, \cite{geier2011decidability} proposes a hybrid planning formalization, TIHTN planning.
For an HTN problem $\HtnProblem = (\HtnDomain,s_0,t_0)$, we say a primitive action sequence $\plan$ is it TIHTN plan, if $\plan$ is executable in $s_0$ and
there is a valid DT $\Dtree$ \emph{w.r.t.} $\HtnProblem$ whose $\yield(\Dtree)$ is not required to be executable in $s_0$ satisfying $\yield(\Dtree)$ is a sub-sequence of $\plan$.
%solution to its corresponding TIHTN planning problem is a primitive action sequence which includes a sub-sequence yielded by a valid decomposition tree
%yielded by a valid decomposition tree whose criterion 2 is modified as: there exist $E' \subseteq E(t)$ such that $(E',\Constraints\!\!|_{E'},\alpha |_{E'}) \cong \tn_m$.
%Intuitively, a TIHTN plan is a primitive action sequence executable in the initial state and includes all primitive tasks obtained by applying methods and inserted primitive tasks.
%\cite{alford2015tighttihtn} gives a progression policy for TIHTN planning and it is not difficult to design a progression-based algorithm to find a TIHTN plan and a decomposition tree which excludes inserted primitive tasks.
\begin{examp}[Example \ref{examp:solution} cont.]\label{examp:tihtn_solution}
If \textsf{plane1} is not at airport A in $s_0$, the DT in Example \ref{examp:solution} is not valid as its plan $\plan_1$ is not executable in $s_0$.
While $\plan_2=\langle${\small\textsf{load};\textsf{drive};\textsf{unload};\textbf{fly};\textsf{load};\textsf{fly};\textsf{unload};\textsf{load};\textsf{drive};\textsf{unload}}$\rangle$ is a TIHTN plan to the problem.
\end{examp}

\subsection{Refining Methods and Completing DTs}
Actually, the inserted tasks in the TIHTN plan are subtask candidates: they provide clues for refining the original methods by adding them as subtasks.
Then, based on a TIHTN plan, we propose the completion profile to refine methods and complete decomposition trees.

Suppose the TIHTN planner outputs a plan $\plan$ and its corresponding decomposition tree $\Dtree$, we use $\insertActions_\plan$ to denote all the inserted tasks in $\plan$.
The TIHTN plan actually is an ordering of primitive tasks and we extend the $\aspprec$ relation of $\Dtree$ by considering the execution order of primitive actions in $\plan$.
To get the compound tasks, we use $\innernodes$ to denote the inner nodes of $\Dtree$.
Next, we show how to link these inserted tasks with the inner nodes $\innernodes$ of the decomposition tree $\Dtree$ to generate a new decomposition tree.
%
%A plan actually in fact is an ordering of primitive actions:
%for a plan $\plan$ and two primitive tasks $t_1,t_2$ in the plan, we use $t_1 \aspprec t_2$ to represent that $t_1$ is performed before $t_2$.
%%Then we define

\begin{definition}%[Completion profiles]
\label{def_complete_profile}
%Given an HTN problem $\P$, let $\plan$ and $\Dtree$ be a TIHTN plan of $\P$ and its decomposition tree.
%We use $\insertActions_\plan$ to denote all the inserted tasks in $\plan$ and use $\decomposeMethods_\Dtree$ to denote all inner nodes in $\Dtree$.
We define a completion profile as a function $\improveFunction\! :\! \insertActions_\plan \longrightarrow \innernodes$, such that for every inserted task $\insertTask \in \insertActions_\plan$
there is not a primitive task $\primTask \in \plan$ where either both $\primTask \! \aspprec\!  \improveFunction(\insertTask)$ and $\insertTask \aspprec \primTask$, or $\improveFunction(\insertTask)  \aspprec \primTask$ and $\primTask  \aspprec  \insertTask$.
\end{definition}

Intuitively, every inserted task is associated with a compound task as its subtask.
%The restriction on ordering constraints requires %that %avoids that the inserted tasks as the subtasks of the corresponding compound task will not violate the ordering constraints between the compound task with other tasks. That is,
Every inserted task is restricted to be performed before the predecessors and after the successors of its corresponding compound task. %which it is associated with.
%
%Formally,
%we use $T_\improveFunction^{(c{,}i{,}k)} = \{(o,s') \mid \improveFunction(o,s)=(c,i,k)\}$\footnote{The identifiers of these inserted tasks should be fresh.} to denote all primitive tasks matching with the decomposition step $(c,i,k)$ \emph{w.r.t.} the completion profile $\improveFunction$. Then for every decomposition step $(c,i,k) \in \decomposeMethods_\aspM$, if $T_\improveFunction^{(c{,}i{,}k)}$ is not empty, we create a completed HTN method $m^c_{k'} = (c,(T_k\cup T_\improveFunction^{(c{,}i{,}k)},{\Constraint_k} \cup {\Constraint'}))$ for the HTN method $m^c_k {=} (c, (T_k,\Constraint_k))$ where $\Constraint' = {\aspprec}|_{T_k \cup T_\improveFunction^{(c{,}i{,}k)}}$.
%%Note that different identifiers $i$ of $c$ cause different completing HTN methods.
%We use $\Methods_{\improveFunction}$ to denote the set of the completed methods \emph{w.r.t.} $\improveFunction$. Then we say $\HtnDomain_{\improveFunction} \!= (\Lvar,\Operators,\Cactions,\Methods\cup \Methods_{\improveFunction})$ is the completion of the domain $\HtnDomain = (\Lvar,\Operators,\Cactions,\Methods)$ by the completion profile $\improveFunction$.

%For two task networks $\tn_1 =(T_1)$
Next, we define how to refine a method by inserting tasks.
%\begin{definition}%[Completions of methods]
%Given an HTN method $m = (c,\tn)$, we define the resulting method of completing $m$ by adding a non-empty task network $\tn_1$ as $m' = (c,\tn \cup \tn_1)$.% where $\tn'=\tn \cup \tn_1$.
%%$T \subset T'$ and $\Constraint \subseteq \Constraint'$.
%%Given an HTN planning domain $\HtnDomain$, we use $\SuperMethods$ to denote the set of all completions of the HTN methods in $\Methods$ for all inserting task networks $\tn_1$ are formed over $\Cactions \cup \Operators$.
%%we say the set of all completions of the HTN methods in $\Methods$ is the completion set of $\Methods$, denoted by $\SuperMethods$.
%\end{definition}
%The HTN method set $\SuperMethods$ is infinite.
%The completion set of $\SuperMethods$ is itself.
A completion profile leads to a set of refined methods by adding the relevant inserted tasks into the original methods. %whose head is associated with.
Formally, for a completion profile $\improveFunction$ and $t\in \innernodes$,%, let $t$ be an inner node in the decomposition tree,
we use $T_\improveFunction^t \!=\! \{\insertTask \mid\! \improveFunction(\insertTask)\!=\!t\}$ %$\footnote{The identifiers $s'$ are fresh.}
to denote all inserted tasks associated with $t$.
%Every inner node in $\Tasksof(\improveFunction)$ means that an incomplete method has been applied to decompose it and the inserted tasks associated with it will be added into the method.
The inserted subtasks with the original subtasks of $t$ compose a new subtask network, written by $\tn_\improveFunction^t = (T_\improveFunction^t,{\aspprec}|_{T_\improveFunction^t},\alpha_\plan)$, where $\alpha_\plan$ which labels the inserted task with ground primitive actions.
Suppose $m = (c, (T_m,\prec_m,\alpha_m))$ is the method of $t$, \emph{i.e.}, $\beta(t)=m$,
for every constant in the inserted primitve actions which occurs in the ground actions associated with the children of $t$ or $\alpha(t)$,
we replace it with its corresponding variable in the unground action in $\tn_m$ and update the function $\alpha_\plan$ to $\alpha'_\plan$.
Then we define the refined method of $m$ as $m_\improveFunction^t \!= \! (c, (T_m \cup T_\improveFunction^t, \prec_m \!\cup\! \aspprec\!|_{T_m \cup T_\improveFunction^t}, \alpha_m \cup \alpha'_\plan))$
%\tn_\improveFunction^t)$
\emph{w.r.t.}\ $\improveFunction$.
We use $\Methods_{\improveFunction}$ to denote the set of refined methods from the completion profile $\improveFunction$.
%When we add the completed methods into the HTN domain $\HtnDomain = (\Lvar,\Operators,\Cactions,\Methods)$, it leads to a new HTN domain $(\Lvar,\Operators,\Cactions,\Methods\cup \Methods_\improveFunction)$, written by $\HtnDomain \aspplus \Methods_{\improveFunction}$.
%Then we say $\HtnDomain_{\improveFunction} \!= (\Lvar,\Operators,\Cactions,\Methods\cup \Methods_{\improveFunction})$ is the completion of the domain $\HtnDomain = (\Lvar,\Operators,\Cactions,\Methods)$ by the completion profile $\improveFunction$.

\begin{examp}[Example \ref{examp:tihtn_solution} cont.]
%For the TIHTN plan $\plan_2$ and the DT in Example \ref{examp:incomplte_methods},
We have a completion profile $\improveFunction$ where $\improveFunction(t_4)\!=\!t_2$. % and $\alpha(t_4)\!=\!\mathsf{fly(plane1,airpA)}$.
The refined method of the original method $m$ is $m_\improveFunction^{t_2}{=}\mathsf{(airShip(?pkg{,}?loc1{,}?loc2)},(T'{,}{\prec}'{,}\alpha'))$ where
\begin{itemize}
\item $T' = \{t_4',t_{21}',t_{22}',t_{23}'\}$;
\item $\prec' = \prec_m \cup \{(t_4',t_{21}'),(t_4',t_{22}'),(t_4',t_{23}')\}$;
\item $\alpha'(t_4') = \mathsf{fly(?plane,?loc1)}$,\\ $\alpha'(t_{21}')= \mathsf{fly(?pkg,?plane,?loc1)}$, etc.
\end{itemize}
\end{examp}

The completion profile actually completes the DT: the inserted tasks are connected with their corresponding inner nodes as their children.
When we add new nodes into the DT, the integrity of ordering constraints will be destroyed.
%To guarantee
%that the ordering constraints about each node is propagated to its children,
%the consistence among the inner nodes and their descendants on the ordering constraints,
To avoid that, we define an operator $\closure$ to complete the ordering constraints.
Formally, for a tree $\Dtree=(T,E)$, we define its closure on the ordering constraint $\prec$ as $\closure(T,E,\prec)$, given by:
%\centerline{
% $\prec \cup  \bigcup_{t\in T}\big((\BEF_{\!{\prec}\!}(t)\times E(t)) \cup (E(t) \times \AFT_{\!{\prec}\!}(t))\big).$
  $$
  \prec \cup  \bigcup_{t\in T}\{(t',ch),(ch,t'')| ch \in E(t), t'\prec t, t\prec t''\}.$$
%}
Intuitively, the closure operation completes the ordering constraints about the children which should be inherited from their parent.

Next, %e show how to complete DTs according to completion profiles.
%\begin{definition}
%Given a decomposition tree $\Dtree = (T,E,\prec,\alpha,\beta)$ and a its completion function $\improveFunction$ \emph{w.r.t.} a plan $\plan$,
we define the completion of DT $\Dtree$ by completion profile $\improveFunction$ \emph{w.r.t.}\ TIHTN plan $\plan$ as $\Dtree_\improveFunction = (T',E',{\prec',}$ $\alpha',\beta')$, which is given by:
%%%-1mm}
\begin{align*}
  T' &:= T \cup  \textstyle \bigcup_{t \in \Tasksof(\improveFunction)} T_\improveFunction^t \\
  E' &:= E \cup  \{(t,st)\mid t\in T, st\in\Tasksof(\tn_\improveFunction^t)\} \\
  %\prec_\improveFunction &= \bigcup\limits_{t\in \Tasksof(\improveFunction)}\big(\aspprec\!|_{T_\improveFunction^t}\cup (\BEF_{\!\prec\!}(t)\times T_\improveFunction^t) \cup (T_\improveFunction^t \times \AFT_{\!\prec\!}(t))\big)\\
  %\prec'\ &:= \prec \cup \prec_\improveFunction \cup \bigcup\limits_{t\in T}\big((\BEF_{\!{\prec} {\cup} {\prec_\improveFunction}\!}(t)\times \ch'(t)) \cup (\ch'(t) \times \AFT_{\!{\prec} {\cup} {\prec_\improveFunction}\!}(t))\big)\\
  \ \prec'\! &:= \closure(T',E',\prec) \cup \textstyle \bigcup_{t\in\Tasksof(\improveFunction)}\aspprec\!\!|_{T^t_\improveFunction \cup E(t)}\\
  \alpha' &:= \alpha \cup \alpha_{\plan}\\
  \beta' &:= (\beta \setminus \{(t,m)\mid t \in \Tasksof(\improveFunction)\}) \cup \{(t,m_\improveFunction^t)\mid t \in \Tasksof(\improveFunction)\}
\end{align*}
where $\Tasksof(\improveFunction)$ is the set of the inner nodes associated by $\improveFunction$.
%\end{definition}
%%%-1mm}

The procedure of completing a DT consists of first connecting the inserted tasks with the inner nodes, then completing the ordering constraints and finally updating the method applied as the refined method.
The DT being completed will satisfy the instance: %, as the following proposition states.

\begin{proposition}\label{prop:complete_tree}
Given an HTN problem $\HtnProblem {=} (\HtnDomain,s_0,t_0)$, let $\plan$ be one of its TIHTN plans, $\Dtree$ be its corresponding DT and $\improveFunction$ be one of their completion profiles. Then the completed DT $\Dtree_{\improveFunction}$ satisfies $(s_0,t_0)$ under the new domain $\HtnDomain \aspplus \Methods_{\improveFunction}$.
\end{proposition}
%%%-1mm}

\begin{proof}
First, we show that $\Dtree_{\improveFunction}$ is a valid DT \emph{w.r.t.}\ $\HtnDomain \aspplus \Methods_{\improveFunction}$.
%Then we check if $\Dtree_{\improveFunction}$ satisfies the three conditions in Def. \ref{def_valid_decomp_tree}.
For every node $t$ in $\Dtree_\improveFunction$ with $\beta'(t)\!=\!(c,\tn_\improveFunction^t)$, i) the function $\alpha$ is not reduced, so $\alpha'(t)\!=\!c$; ii) the edges between the task $t$ and its inserted tasks $T_\improveFunction^t$ are added, so the task network induced by its children is a grounding of $m_\improveFunction^t$; iii) $\closure(T',E',\prec)$ guarantees that all ordering constraints of $t$ are propagated to the inserted tasks and ${\aspprec}|_{T^t_\improveFunction \cup E(t)}$ only introduces the ordering constraints among the subtasks in the same method, so conditions 3. and 4. are satisfied;
iv) as the completion profile guarantees that no contradict pair about $\aspprec$ is introduced, condition 5. is satisfied.

%Next, we prove $\Dtree_{\improveFunction}$ to satisfy the instance $(s_0,t_0)$.
Without removing nodes, the root of $\Dtree_{\improveFunction}$ is still $t_0$.
As the plan $\yield(\Dtree_{\improveFunction})$ is the TIHTN plan $\plan$ executable in $s_0$, $\Dtree_{\improveFunction}$ satisfies the instance $(s_0,t_0)$.
\end{proof}

%The TIHTN plans help to find more ways to decompose the initial task network via completion profiles.
%In particular,
When an HTN problem has incomplete methods, the completion profile offers a way to improve the HTN domain: %to get an HTN plan.%make it become solvable:

\begin{theorem}\label{thm:repair_htn}
If an HTN problem $\P{=}(\HtnDomain,s_0,t_0)$
has a TIHTN plan but no solution, then there is a completion profile $\improveFunction$ where the HTN problem $\P'=(\HtnDomain \aspplus \Methods_\improveFunction,s_0,t_0)$ is solvable.
%$\TIHTNsolutions(\P) \neq \emptyset$, then $\HTNsolutions(\HtnDomain_\improveFunction,s_0,\tn_I) \neq \emptyset$ where $\improveFunction$ is a completion profile \emph{w.r.t.} a TIHTN plan of $\P$.
\end{theorem}

%%%-2mm}
\begin{proof}
Straightforward.
%Let $\plan$ be a TIHTN plan of $\P$ with its DT $\Dtree$. % and $\improveFunction$ be a completion profile.
%Suppose $\improveFunction$ is a completion profile \emph{w.r.t.}\ $\plan$ and $\Dtree$.
%By Proposition $\ref{prop:complete_tree}$, $\Dtree_\improveFunction$ satisfies the instance $(s_0,t_0)$ under the new domain $\HtnDomain \aspplus \Methods_{\improveFunction}$.
%So, $\plan$ is a solution of the HTN problem $\P'$.
\end{proof}

\subsection{Prioritized Preferences}

%Similar to the preferences in ontology repair \cite{DBLP:conf/aaai/BienvenuBG14}, we consider five kinds of preferences on HTN method sets.

%\begin{definition}
%Given an HTN problem $\HtnProblem=(\Lvar,\Operators,\Cactions,\Methods,s_I,\tn_I)$, we define a preference $\prefer$ as a preorder on $\mathcal{P}(\Methods)$.
%\end{definition}
To formalize the phenomenon that the missing of subtasks happens more likely on some methods than other methods,
we consider a prioritization on the methods. %: the methods with a lower priority should be completed more easily than those methods with a higher priority.

%\textbf{Set inclusion $(\sqsubseteq)$.} Generally, HTN methods are supposed to be changed as minimally as possible. Then we consider the preference based on the set inclusion on the revised methods:
%for two method sets $\Methods_1,\Methods_2\subseteq \Methods$, if $\Methods_1 \subseteq \Methods_2$, then $\Methods_1 \sqsubseteq \Methods_2$.
%
%\textbf{Cardinality $(\leq)$.} We also consider the minimal change in terms of cardinality: for $\Methods_1,\Methods_2\subseteq \Methods$, if $|\Methods_1| \leq |\Methods_2|$ then $\Methods_1 \leq \Methods_2$.

%\noindent\textbf{Prioritized preference $(\leq_P)$.} %There always exists a priority among HTN methods when the user designs the HTN domain. So we also consider the preference based on the priority.
Given a method set $\Methods$, we define a prioritization as a partition on it: $P {=} \langle P_1,...,P_n \rangle$ where $\bigcup_{1\leq j \leq n} P_j =\Methods$. %and we called the partition a .
Intuitively, the HTN methods in $P_i$ have a higher priority to be refined than those in $P_j$ if $i > j$.
We further consider the prioritized preference $\leq_P$ in terms of cardinality: %of the methods refined.
%%Here we use $|S|$ to denote the cardinality of the set $S$.
%
%%\begin{definition}
%Given a prioritization $P {=} \langle P_1,...,P_n \rangle$ of $\Methods$, we consider the prioritized preference $\leq_P$ as follows:
for $\Methods_1,\Methods_2\subseteq \Methods$,
%   \begin{itemize}
%    \item  Prioritized Set inclusion ($\subseteq_P$): if $\Methods_1 \cap P_i = \Methods_2 \cap P_i$ for every $1 \leq i \leq n$, there is some $1 \leq i \leq n$ such that $\Methods_1 \cap P_i \subseteq \Methods_2 \cap P_i$ and for all $1 \leq j < i, \Methods_1 \cap P_j = \Methods_2 \cap P_j$ then $\Methods_1 \sqsubseteq_P \Methods_2$
%    \item  Prioritized cardinality ($\leq_P$):  if $|\Methods_1 \cap P_i| = |\Methods_2 \cap P_i|$ for every $1 \leq i \leq n$, there is some $1 \leq i \leq n$ such that $|\Methods_1 \cap P_i| \leq |\Methods_2 \cap P_i|$ and for all $1 \leq j < i,  |\Methods_1 \cap P_j| = |\Methods_2 \cap P_j|$ then $\Methods_1 \leq_P \Methods_2$
%  \end{itemize}
%Prioritized cardinality ($\leq_P$):
if %$|\Methods_1 \cap P_i| = |\Methods_2 \cap P_i|$ for every $1 \leq i \leq n$,
there is some $1 \!\leq i\! \leq n$ such that
$|\Methods_1 \!\cap\! P_i| \!\leq\! |\Methods_2 \!\cap\! P_i|$ and that
$\forall\ 1 \!\leq\! j \!<\! i,  |\Methods_1 \!\cap\! P_j| \!=\! |\Methods_2 \!\cap\! P_j|$,
then we write $\Methods_1 \!\leq_P\! \Methods_2$.
%\end{definition}
%\textbf{Weights($\leq_w$).} Every HTN method has a cost to be changed and we consider the HTN method set is more preferred if its sum of the weights is less. Given a weight assignment $w:\Methods \longrightarrow \mathbb{N}$, we define $\Methods_1 \leq_w \Methods_2$ if $\Sigma_{m\in\Methods_1}w(m) \leq \Sigma_{m\in\Methods_2}w(m)$.
We say $\Methods_1$ is strictly preferred over $\Methods_2$ \emph{w.r.t.}\ $P$, written by $\Methods_1 \!<_P\! \Methods_2$, if $\Methods_1 \!\prefer\! \Methods_2$ and $\Methods_2 \!\not\prefer\! \Methods_1$.

Generally, the prioritization comes from the confidences of domain experts on methods: the method believed to lack subtasks more likely to have a higher priority to be refined.
In particular, there exists a class of HTN domains where actions can be stratified according to the decomposition hierarchy \cite{erol1996complexity,DBLP:conf/socs/AlfordSKN12}.
We found an interesting phenomenon that the more detailed tasks are more sensitive to these situations and more easily to be thoughtless.
%In practice, for the more detailed compound actions, the case that subtasks are missed occurs more possibly, because in the more elaborated levels, more possible situations need to be considered.
In this case, we assume that an action is more abstract than its subtasks and we consider a preference in terms of a stratum-based prioritization: the more abstract actions have a lower priority to be refined.
\subsection{Preferred Completion Profiles}
Generally, we hope to find a completion profile changing the original methods minimally under the prioritized preference.

%Before giving the definition of the preferred completion profile,
We first define some notations:
for a refined method $m^t_\improveFunction$, we use $\inverseComplete(m^t_\improveFunction)$ to denote its original method $m$. %and $\insertTasksof_t(m^t_\improveFunction)$ to denote the inserted subtasks \emph{w.r.t.} $t$, \emph{i.e.}, $\insertTasksof_t(m^t_\improveFunction) = T^t_\improveFunction$.
%$\inverseComplete$ as the inverse operation of completion $\complete$, to find the original method before completion.
%Formally, for a completion $m'$, we use $\inverseComplete(m')$ to denote the original method $m$ such that $m' = \complete(m)$.
For a refined method set $\Methods'$, we use $\inverseComplete(\Methods')$ to denote all the original methods of the refined methods in $\Methods'$, \emph{i.e.}, $\inverseComplete(\Methods')\!=\!\{m \!\in\! \Methods | \ m \!=\! \inverseComplete(m'), m' \in \Methods'\}$.
Note that several completions may be associated with the same HTN method.
For two HTN methods $m'_1$ and $m'_2$, if $\inverseComplete(m'_1) = \inverseComplete(m'_2)$, we say $m'_1$ and $m'_2$ are \emph{homologous}.

%Next we define the preferred completion profile which the methods minimally with the prioritized preference.

\begin{definition}%[Preferred completion profiles]
Given a TIHTN plan and its corresponding DT, a completion profile $\improveFunction$ is preferred \emph{w.r.t.}\ preference $P$ if there is not a completion profile $\improveFunction'$ s.t. $\inverseComplete(\Methods_{\improveFunction'}) \strictPrefer \inverseComplete(\Methods_{\improveFunction})$. %(\emph{i.e.}, $\inverseComplete(\Methods_{\improveFunction'}) \prefer \inverseComplete(\Methods_{\improveFunction})$ and $\inverseComplete(\Methods_{\improveFunction}) \not\prefer \inverseComplete(\Methods_{\improveFunction'})$).
\end{definition}
Intuitively, the preferred completion profile refines methods minimally under the prioritized preference.
%the learned methods by completion is minimal under the prioritized preference.
%

Next, we will show how to find the preferred completion profile, as shown in Algorithm \ref{algor:complete}.
First, we consider all inserted tasks in the plan as unlabelled (line 1).
Then we scan all inner nodes from the nodes with a method of higher priority to the nodes with a method of lower priority (line 2-3).
Next, for an inner node, we find the set of candidate subtasks $\SubTCandidates$ from the inserted tasks, which do not violate the ordering constraints if they were inserted as its subtasks (line 5).
More specifically, for the inner node $t$, the inserted tasks which are executed between the last task required to be executed ahead of $t$ and the first task required to be after $t$, are allowed to be added as subtasks of $t$.
%When the plan is an ordering of primitive tasks, we can easily find the first primitive task in $\AFT_{\aspprec}(t)$ on the plan, denoted by $\FirstAFT t$, and the last primitive task in $\BEF_{\aspprec}(t)$, denoted by $\LastBEF t$.
%Then we define the subtasks candidate set $\SubTCandidates$ of $t$ as the set of unlabelled inserted tasks which are executed between $\BEF_{\aspprec}(t)$ and $\FirstAFT t$.
According to the total order `$\aspprec$' in the DT,
we define the subtasks candidate set $\SubTCandidates$ of $t$ as the set of the unlabelled inserted tasks between the last predecessor of $t$ and the first successor of $t$.
Finally, we associate all tasks in the subtask candidate set to $t$ (line 5) and label them as subtasks (line 6).
When all inserted tasks are labeled, it returns a preferred completion profile.
It must terminate and the worst case is that the inserted tasks are associated with the root task.

Algorithm \ref{algor:complete} only scans the inner nodes of the DT once and searching the subtask candidate set can be done in linear time, so the algorithm terminates in polynomial time.

%Given a decomposition tree $\Dtree$, we use $\BEF_t$ to denote the set of primitive tasks which are required to be before the task $t$, \emph{i.e.}, $\BEF(t)=\{\primTask| \primTask \prec_\Dtree t\}$. Symmetrically, we use $\AFT_t$ to denote the set of primitive tasks required to be after $t$, \emph{i.e.}, $\AFT_t=\{\primTask | t \prec_\Dtree \primTask\}$. For a plan $\plan$ corresponding to $\Dtree$, we use $\LastBEF t$ to denote the last primitive tasks within the primitive tasks required to be before the task $t$ on the plan $\plan$ and $\FirstAFT t$ to denote the first primitive tasks within the primitive tasks required to be after $t$ on the plan $\plan$.
%Therefore, the inserted tasks between $\LastBEF t$ and  $\FirstAFT t$ are subtasks candidates of the task $t$ and we use $\SubTCandidates$ to denote these inserted tasks.

%%%-4mm}
\begin{algorithm}%[!htp]
\setlength{\abovecaptionskip}{0mm}   %调整图片标题与图距离
\setlength{\belowcaptionskip}{-5mm}
  \caption{$\textsc{Complete}(\plan,\Dtree,P)$}\label{algor:complete}
\Input {A TIHTN plan $\plan$, its DT $\Dtree$ and a prioritization $P=(P_1,...,P_n)$ on $\Methods$}
\Output {A completion profile $\improveFunction$}
%\STATE $ \leftarrow \emptyset$\\
$I \gets I_\plan$\;
\For {$j \leftarrow n$ to $1$}
    {
    \For {each $t \in \innernodes$ s.t. $\beta(t) \in P_j$}
        {
        \If {$\insertActions \neq \emptyset$}
            {
            for every $t' \in \SubTCandidates \cap \insertActions$, %into $T_\improveFunction(t)$\;
            set $\improveFunction(t') = t$\;
            $\insertActions \gets \insertActions \setminus \SubTCandidates$\;
            }
        }
}
\textbf{return} $\improveFunction$
%  \end{algorithmic}
%%%-2mm}
\end{algorithm}

Actually, to find a preferred completion profile, we only need to scan the inner nodes in the DT according to the preference and link appropriate inserted tasks with inner nodes, which can be done in polynomial time.

%In this case, we assume that an action is more abstract than its subtasks and has a lower priority to be refined.

%%%-3mm}
\section{Refining Methods from Instances}
%In this paper, we assume that the original methods are kept as they come from the expert knowledge and they are sound in some situations.
%By Theorem \ref{thm:repair_htn}, the introduction of new methods helps to solve the problems.
%Therefore, to learn HTN methods, we create new methods by completing the original methods and add them into the HTN domain.
As stated above, we focus on keeping the original methods coming from domain experts and consider adding the refined methods into the original domain.
For an HTN domain $\HtnDomain = (\Lvar,\Operators,\Cactions,\Methods)$ and a method set $\Methods'$, we use  $\HtnDomain \aspplus \Methods'=(\Lvar,\Operators,\Cactions,\Methods\cup \Methods')$ to denote the resulting domain by adding $\Methods'$ into $\HtnDomain$.
An HTN method refining problem is defined as a tuple $(\HtnDomain,\Instances)$ where $\HtnDomain$ is an HTN domain and $\Instances$ is a set of instances.

However, an excess of methods introduced may slow down problem-solving significantly, as there are excessive choices to decompose tasks.
So, we hope the refined methods to be as minimal as possible.
%\noindent\textbf{Method Refinement.} %In this paper, we assume that the original methods are kept as they come from the expert knowledge and they are sound in some situations.
%So, we only consider adding methods into the original domain.
Then we define a solution of the HTN method refining problem $(\HtnDomain,\Instances)$ is a set of refined methods $\Methods'$ which should satisfy:
\begin{itemize}
\item all instances in the set $\Instances$ are solvable under $\HtnDomain \aspplus \Methods'$;
%\item the method set learned $\Methods'$
\item the refined method set $\Methods'$ is as minimal as possible;
\item the refined methods in $\Methods'$ have as few inserted subtasks as possible.
\end{itemize}

Given a set of HTN instances, we first compute a TIHTN plan and a DT for each instance and then compute the preferred completion profile.
Indeed, the completion profiles for various instances induce many different refined methods with the same head which possibly handel the same situation.%and generate similar executable plans.
%The vast increase in the number of methods will slow down the problem-solving significantly and
Such refined methods are redundant because they can be replaced by other methods.

To compute the minimal set, we define a method substitution operation for a decomposition tree.
More specifically, for a DT $\Dtree$ and two homologous refined methods $m'_1$ and $m'_2$, we use $\mathsf{sub}(\Dtree,m'_1,m'_2)$ to denote the resulting DT obtained by replacing every subtree induced by $m'_1$ with a subtree by $m'_2$ and completing the closure of the ordering constraints.
%If $\yield(\mathsf{sub}(\Dtree,m_1,m_2))$ is executable in the corresponding initial state, then $\mathsf{sub}(\Dtree,m_1,m_2)$ satisfies the instance.
Then for an instance $(s_0,t_0)$, if $\Dtree$ satisfies it and $\yield(\mathsf{sub}(\Dtree,m'_1,m'_2))$ is executable in $s_0$, then $\mathsf{sub}(\Dtree,m'_1,m'_2)$ also satisfies it.
In other words, for this instance, the refined method $m'_1$ is replaceable by $m'_2$.

Then we generalize the method substitution operation into sets:
given a DT set $\overline{\Dtree}$ and two refined method sets $\Methods'_1,\Methods'_2$,
we use $\mathsf{sub}(\overline{\Dtree},\Methods'_1,\Methods'_2)$ to denote the set of the DTs that substitutes every method in $\Methods'_1$ with certain method in $\Methods'_2$.
%With the substitution operation, we can reduce the refined methods: %, as the following proposition states.
If each resulting DT still satisfies the corresponding instance, we say $\Methods'_1$ is replaceable by $\Methods'_2$.
Given a set $\Methods'$ of refined methods, for the minimality, we need to compute the minimal subset $\Methods''$ of $\Methods'$ which are not replaceable by any strict subset of $\Methods''$.

Indeed, it is difficult to find the minimal method set under the replaceability relation between refined methods, as all subsets of the refined method set require to be considered.
Fortunately, the prioritized preference indicates what methods should be refined in a higher priority.
The prioritized preference can be extended to the refined methods: $m'{\in} P_j$ if $\inverseComplete(m')\in P_j$.
%Next we define the minimal subset with the prioritized preference in terms of cardinality.
Then we seek for the sub-optimal solution
by computing the minimal subset \emph{w.r.t.} the replaceability relation under the refined methods with the same priority, which reduces the searching space significantly.

%\begin{definition}
%Given a method set $\Methods'$ and its prioritization $P$,
%a subset $\Methods'_0$ of $\Methods'$ is the minimal set \emph{w.r.t.}\ $P$ if there is not a subset $\Methods'_1$ of $\Methods'$ such that $\Methods'_1 <_P \Methods'_0$. %(\emph{i.e.} $\Methods_1 \prefer \Methods_0$ and $\Methods_0 \not\prefer \Methods_1$).
%\end{definition}

Next, we give an algorithm to refine methods for a set of HTN problems and a given prioritization, as shown in Algorithm \ref{algor:method_learn}.
The framework consists of two main components: the first iteration for refining methods (line 2-8) and the second iteration for reducing refined methods (line 9-11).

We developed the TIHTN planning approach in \cite{alford2015tighttihtn} (noted \textsc{HPlan}).
In order to pursue as few inserted subtasks as possible in refined methods, we exploited a breadth-first strategy to search inserted tasks.

In the first iteration, it first invokes \textsc{HPlan} to compute TIHTN plan and the corresponding DTs (line 3) and then computes preferred completion profiles (line 4) by \textsc{Complete}.
According to these completion profiles, it completes these DTs and constructs a set of refined methods.
%To reduce repetitive refined methods, if the constants in the inserted subtasks occur in the variables of the original methods, we use the corresponding variables to replace them in the inserted subtasks when aggregating refined methods (line 8).

In the second iteration, we use a greedy strategy to find the minimal set: the refined methods with lower priority are reduced first, which is the opposite against the procedure of searching the preferred completion profile. %in Algorithm \ref{algor:complete}.
Here we use $P_j[\Methods']$ to denote the refined methods in $\Methods'$ with the priority $P_j$.
The algorithm enumerates the power set of $P_j[\Methods']$ and computes the minimal subset $\Methods'_j$ of $P_j[\Methods']$ \emph{w.r.t.} the replaceability relation under $P_j[\Methods']$ and the union of the minimal subsets with lower priorities.
%By Proposition \ref{prop:method_set_substitution}, the refine methods in $P_j[\Methods']$ can be replaced by $\Methods'_j$ for every instance and $\Methods'_j$ is minimal in the priority $P_j$.
%From $P_n$ to $P_1$, we get a

%
%take a breadth-first strategy to find the inserted tasks when computing TIHTN plans (line 7 in Algorithm \ref{algor:hplan}).

\begin{algorithm}[!htp]
  \caption{$\textsc{MethodRefine}(\HtnDomain,\Instances,P)$}\label{algor:method_learn}
\Input {An HTN domain $\HtnDomain$, an instance set $\Instances$ and a prioritization $P=(P_1,...,P_n)$ on $\Methods$}
\Output {A set of refined methods $\Methods''$}
%\STATE $\HtnDomain' \leftarrow \HtnDomain$
 $\Methods' \leftarrow \Methods'' \leftarrow \emptyset$;\hspace{1.5em}
 $\overline{\Dtree} \leftarrow \emptyset$\;
% \% compute all instances' plans, DTs and completion profiles
\For {each $i$ in $\Instances$}
    {
      compute a plan and DT $(\plan^i,\Dtree^i)=\textsc{HPlan}(\HtnDomain,i)$\;
      $\improveFunction^i = \textsc{Complete}(\plan^i,\Dtree^i,P)$\;
      complete the DT $\Dtree^i$ to $\Dtree^i_{\improveFunction}$ by $\improveFunction^i$\;
      $\overline{\Dtree} \leftarrow \overline{\Dtree} \cup \Dtree^i_{\improveFunction}$\;
      construct a new method set $\Methods^i_\improveFunction$ from $\improveFunction^i$\;
      $\Methods' \leftarrow \Methods' \cup \Methods^i_\improveFunction$\;
    }
%\STATE \% find the optimal new method set
%replace some constants in $m'\in \Methods'$ by variables\;
\For {$j \leftarrow 1$ to $n$}
    {
     compute the minimal subset $\Methods'_j$ \emph{w.r.t.} the replaceability relation under $P_j[\Methods'] \cup \Method''$\;
%    \If {$\mathsf{sub}(\overline{\Dtree},\Methods',\bigwedge \Methods'_j)$ correspond to a plan}
%        { $\Methods'_j \leftarrow \bigwedge \Methods'_j$}
%    update $\overline{\Dtree}$ by the substituting trees $\mathsf{sub}(\overline{\Dtree},\Methods',\Methods'_j)$\;
 $\Methods'' \gets \Methods'' \cup \Methods'_j$\;
    } %into $\HtnDomain$ to be $\HtnDomain'$\;
\textbf{return} $\Methods''$
\end{algorithm}

In fact, due to the greedy strategy, our approach cannot guarantee criterion 2 and 3,
but must satisfy criterion 1:

%Next we show the soundness of our algorithm:
%%%-1mm}
\begin{theorem}
Suppose $\Methods''$ is the method set refined by \textsc{MethodRefine}$(\HtnDomain,\Instances,P)$,
if every instance in $\Instances$ has a TIHTN plan under the domain $\HtnDomain$, then it is also solvable under the domain $\HtnDomain \aspplus \Methods'$.% where %$\Methods'$ is the method set learned by
%$\Methods'=$ \textsc{MethodRefine}$(\HtnDomain,\Instances,P)$.
\end{theorem}

\begin{proof}
As every instance has a TIHTN plan, by Proposition \ref{prop:complete_tree}, there exists a set of DTs $\overline{\Dtree}$, each of which satisfies each instance \emph{w.r.t.}\ the domain $\HtnDomain \aspplus \Methods'$ where $\Methods'$ is a method set obtained via completion profiles. % (line 2-8 in Algorithm \ref{algor:method_learn}). %Then $\Methods'' {\subseteq} \Methods'$.
As in the second iteration, the algorithm keeps the satisfiability of the instances, $\Method'$ is replaceable by $\Methods''$ and each DT in $\mathsf{sub}(\overline{\Dtree},\Methods'',\Methods')$ satisfies its corresponding instance in $\Instances$. Thus, every instance is solvable \emph{w.r.t.}\ the new domain $\HtnDomain \aspplus \Methods'$.
\end{proof}

\begin{figure*}[!htp]
\centering
\subfigure[The Logistics Domain]{\includegraphics[width=0.33\textwidth]{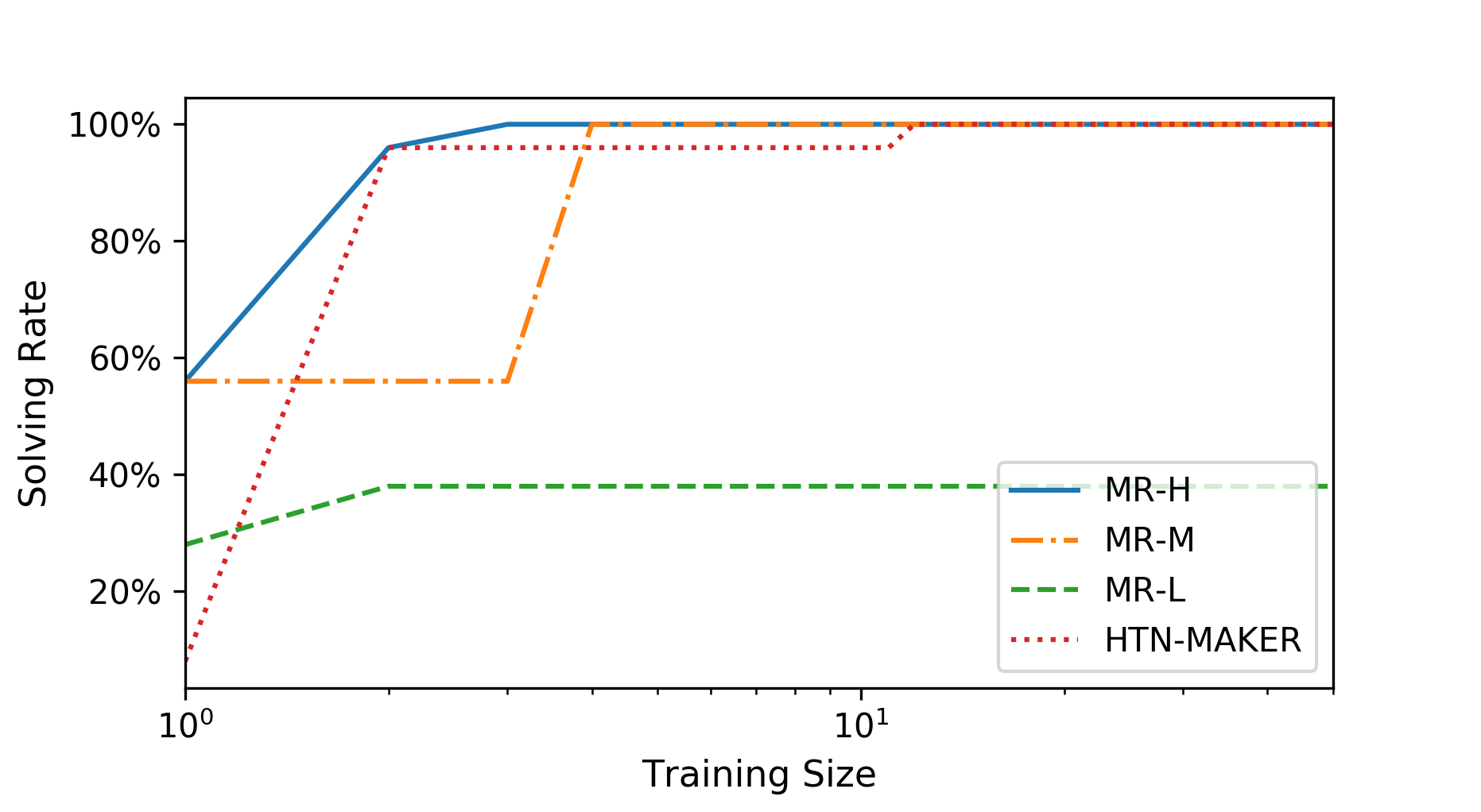}}
%\hspace{1cm}
\subfigure[The Satellite Domain]{\includegraphics[width=0.33\textwidth]{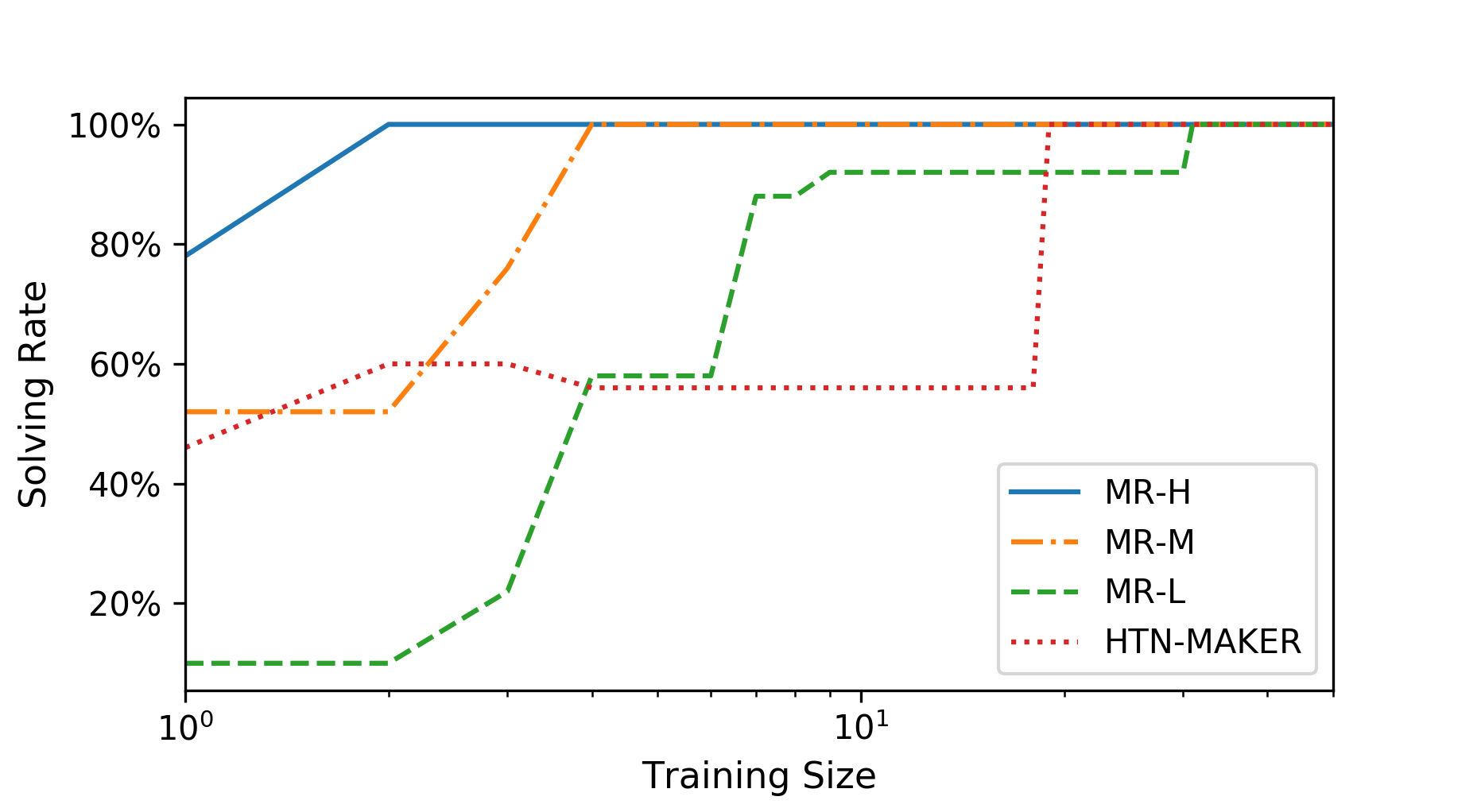}}
\subfigure[The Blocks-World Domain]{\includegraphics[width=0.33\textwidth]{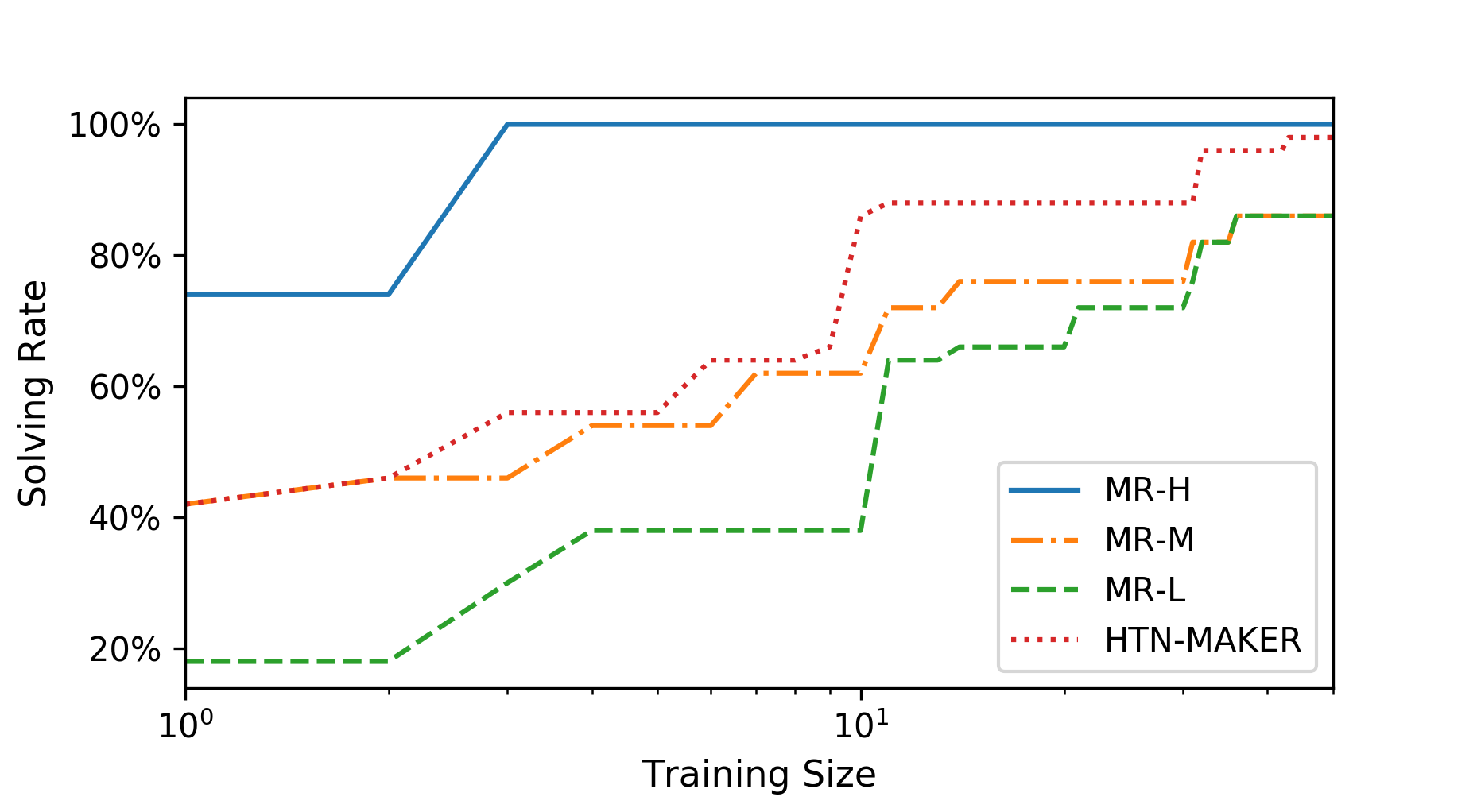}}
%%%-0.5cm}
\caption{\small Experimental Results on the Solving Rate of \textsc{MethodRefine} with Different Domain Incompleteness and HTN-MAKER}\label{fig:experiment1}
%%%-0.5cm}
\end{figure*}

\begin{figure*}[!htp]
%%%-2mm}
%\setlength{\abovecaptionskip}{1mm}   %调整图片标题与图距离
%\setlength{\belowcaptionskip}{-3mm}
\centering
\subfigure[The Logistics Domain]{\includegraphics[width=0.33\textwidth]{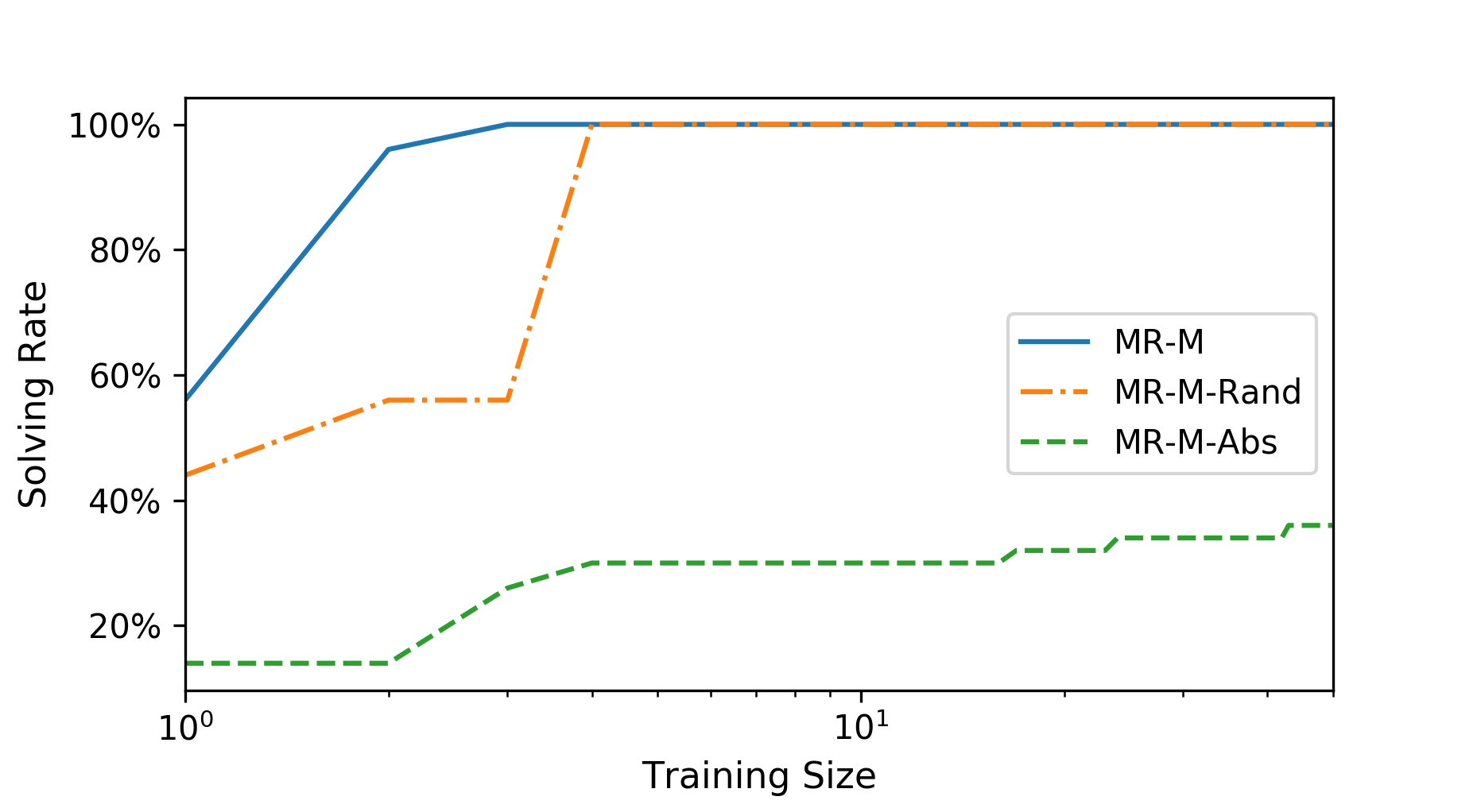}}
%\hspace{1cm}
\subfigure[The Satellite Domain]{\includegraphics[width=0.33\textwidth]{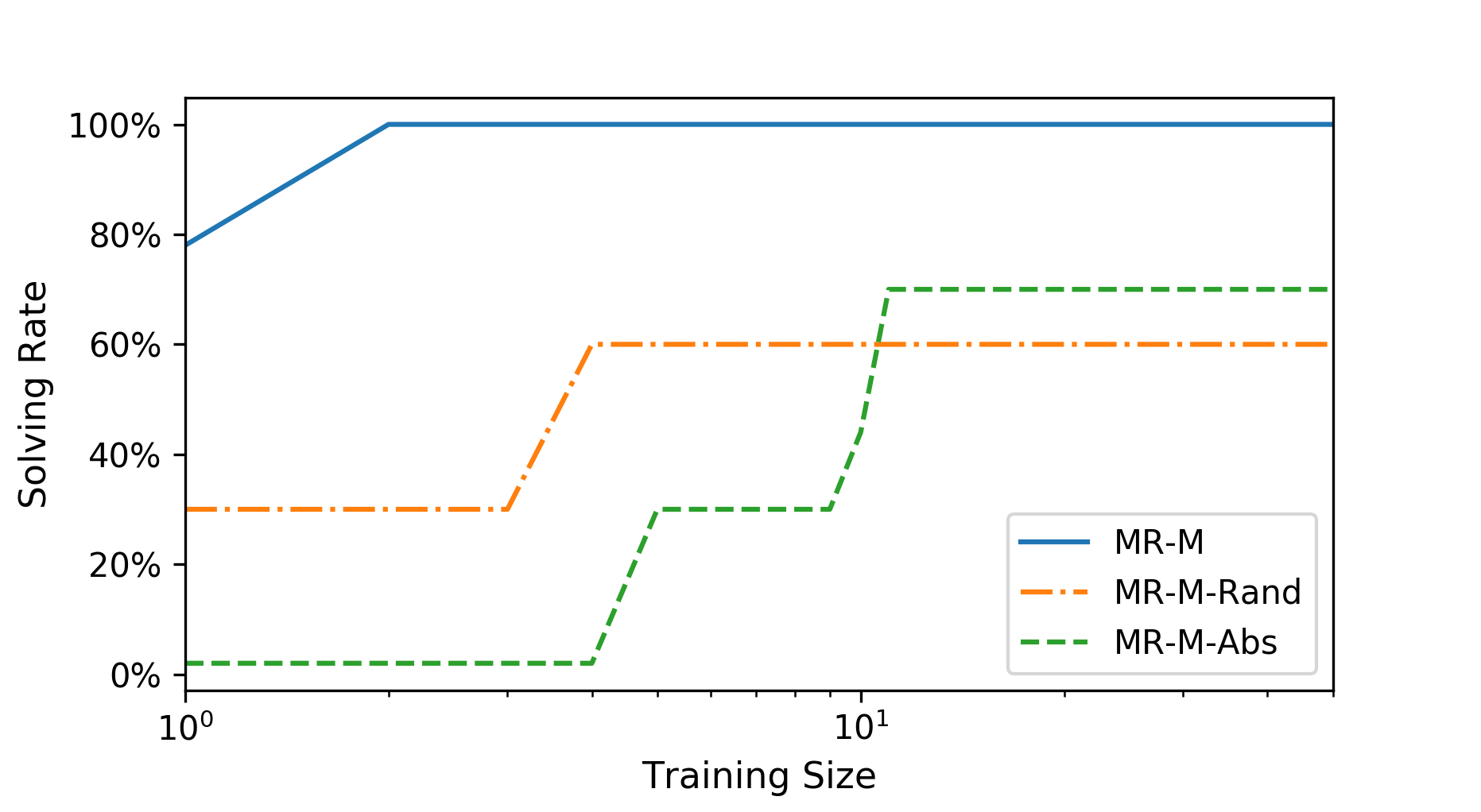}}
\subfigure[The Blocks-World Domain]{\includegraphics[width=0.33\textwidth]{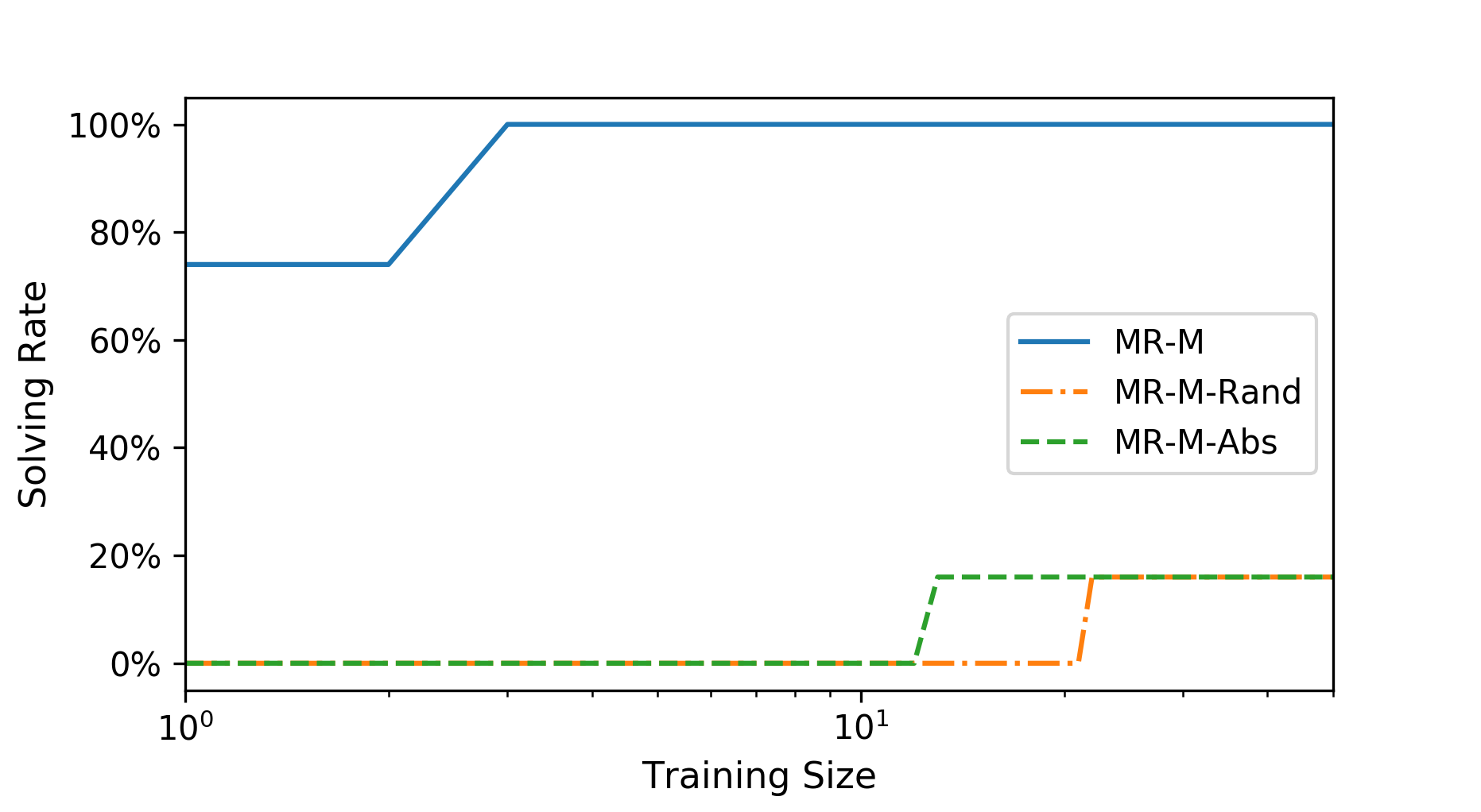}}
%%%-0.5cm}
\caption{\small Experimental Results on the Solving Rate of \textsc{MethodRefine} with Different Preferences}\label{fig:experiment2}
%%%-0.5cm}
\end{figure*}

\section{Experimental Analysis}
%We have implemented \textsc{MethodRefine} based on Python.
In this section, we evaluate \textsc{MethodRefine}\footnote{Available in \url{https://github.com/sysulic/MethodRefine}} in three well-known planning domains %on three original different incomplete method sets,
comparing with HTN-MAKER\footnote{\url{http://www.cse.lehigh.edu/InSyTe/HTN-MAKER/}} on the ability of solving new instances. %\footnote{More experiment details are available on \\ {\scriptsize\url{http://tinyurl.com/IJCAI19-230}}.}. % We also evaluate \textsc{MethodRefine} .

We consider the three domains which HTN-MAKER uses: Logistics, Satellite, and Blocks-world.
We first get the problem generators from  International Planning Competition website\footnote{\url{http://ipc02.icaps-conference.org/}} and randomly generate 100 instances for each domain and take 50 instances as the training set and 50 instances as the testing set.
We run \textsc{MethodRefine} and HTN-MAKER with 50 instances growingly as input and obtain different learned method sets from these two approaches.
A testing instance is considered as solved, if its goal is achieved by a plan computed under the learned HTN method set via an HTN planner.
For HTN-MAKER we use the HTN planner from their website$^2$ and for our approach \textsc{MethodRefine} we still use our TIHTN planner with forbidding task insertion.
%The time bound of the HTN planner is set to 3600 seconds.
Experiments are run on the 3.00 GHz Intel i5-7400 with 8 GB RAM with a cutoff time of one hour.
In order to check if an instance is solved, we add a verifying action whose precondition is the goal and whose effect is empty in the last subtask of the initial task.
The learning performance is measured via the proportion of the solved instances on the testing instances, which is called \emph{solving rate}.

\textbf{Experimental results on comparing different domain incompleteness.}
First, we evaluate the influence of the different incompleteness of the given method sets on the solving rate.
To simulate the incomplete method set as the input of \textsc{MethodRefine}, we take the HTN domain descriptions in the website\footnote{\url{https://www.cs.umd.edu/projects/shop/}} of SHOP2 HTN planner, and remove different sets of subtasks from these domains. Then we consider three removal cases: 1) remove one primitive task from each method (if exists), with meaning the high completeness, noted by MR-H; 2) remove two primitive tasks from each method (if exists), noted by MR-M, with meaning the middle completeness; 3) remove one more compound task in some method of MR-L, noted by MR-L, with meaning the low completeness.
%The learning performance is measured by the percentage of the testing problems solved by an HTN solver with the learned methods.
%%\noindent\textbf{Logistics}
%Taking the Logistics domain \cite{veloso1992learning} describes that packages are to be delivered with the truck transportation within the same city and the
%plane transportation among different cities, which
%is described briefly in Figure \ref{fig:example_of_incomplete_method}.
%The 100 instances are generated with the number of packages varied from 1 to 5.
Taking the method set shown in Figure \ref{fig:example_of_incomplete_method} as example,
for MR-H, we remove the first \textsf{drive} and the first \text{fly} in the methods \textsf{cityShip} and \textsf{airShip}, respectively,
while for MR-M, we remove all \textsf{drive} and \text{fly} in the methods.
For MR-L, the first \textsf{cityShip} is additionally removed from the method of \textsf{ship} based on the MR-M setting.
As these domain are stratifiable, we use the stratum-based prioritized preference as the input of \textsc{MethodRefine}.
%
%%\noindent\textbf{Satellite}
%The Satellite domain formalizes the task of getting images via the camera instruments on the satellites.
%The task is decomposed into activating the instrument, turning to the target and taking an image.
%%The task is decomposed into activating the instruments \textsf{activate}, turning to the target \textsf{turn} and taking an image \textsf{takeImage}.
%%The action \textsf{activate} is decomposed into switching off other instruments, switching on and calibrating the instrument which is further decomposed into \textsf{turn} and \textsf{calibrate}
%The 100 instances are generated with the number of instruments varied from 2 to 5.
%%The setting of MR-H, MR-M, and MR-L are designed as above.
%
%%\noindent\textbf{Block world}
%The Blocks-World domain simulates the scenario that there is a robot arm to move blocks on the table.
%The 100 instances are generated with the initial situation varied and the initial task of stacking 5 blocks into a column.

The experimental results are shown in Figure \ref{fig:experiment1}.
It demonstrates that the more complete the domain is, the better the learning performance is.
Generally, the solving rate increases along with the training set growing. %, which does not violate Proposition \ref{prop:monotonic_instances}.
For the Logistics and Satellite domains, in the settings of MR-H and MR-M, \textsc{MethodRefine} learns the necessary methods to solve all testing problems from a few instances.
It is because the structure of these two domains is relatively straightforward and the DTs still can be constructed by the incomplete method sets.
In the MR-L setting, the compound action removed in the Logistics Domain, \textsf{cityShip}, contains more arguments, making the learned methods become more case-specific, which cannot contribute to other instances.

MR-H eventually learns 2 methods which already effectively solve all testing instances on the three domains.
While MR-M and MR-L learn more than 10 methods in the Blocks-World domain but still fail to achieve the full solving rate.
%For the Blocks-World domain, except for the
%The reason is that there are a few special instances which are significantly different from the training instances, resulting in that the learned methods hardly suit these special testing instances.

\textbf{Experimental results on comparing \textsc{MethodRefine} against HTN-MAKER.}
%We compare our approach against HTN-MAKER on the learning performance in and show the
%Figure \ref{fig:experiment1} also shows the learning performances of our approach and HTN-MAKER in the three domains.
From Figure \ref{fig:experiment1}, we observe that HTN-MAKER learns methods less slowly than \textsc{MethodRefine} with MR-H setting.
Comparing with the other two settings, HTN-MAKER is superior on the Blocks-World domain but inferior on the Logistics and Satellite domain.
In the Satellite domain, after training a number of instances, HTN-MAKER exceeds memory limitation when solving some instances, causing a sudden drop in its curve.
It is because HTN-MAKER learns a method which causes an infinitely recursive decomposition.
It never occurs in our approach because we only learn methods from acyclic decomposition trees.
But they are solved when suitable methods are learned with the training set growing.
HTN-MAKER finally learns 87, 23, 92 methods in the Logistics, Satellite, Blocks-World domain, respectively. Comparing with our approach, it learns more methods but many of which are redundant.
%\textbf{Experimental results on the performances of \textsc{MethodRefine} with different incompleteness of domains.}
%As pointed out in \cite{DBLP:conf/aaai/HoggMK08}, ``goals are much more inter-dependent'', which makes it difficult
%
%%Figure \ref{fig:experiment} also shows that
%The experiment results also show that our approach learns significantly fewer methods than HTN-MAKER. %in the settings of MR-H and MR-M. %While in the MR-L setting,
%For the whole training set, our approach only generates 17, 13, 13 methods for each domain even in the MR-L setting, while HTN-MAKER generates 75, 24, 107 methods finally.
%
%2)~To evaluate the stability, we also evaluate the methods by taking the training instances in the reversed order, denoted by MR-H-r and HTN-MAKER-r in Figure \ref{fig:experiment}, respectively.
%It is shown that HTN-MAKER is sensitive to the order of the input instances: in the reversed order, its convergence rate differs from that of the original order.
%While our approach takes the instances into account from a global perspective and returns the same set of methods, despite the order of the instances trained.

\textbf{Experimental results on the performances of \textsc{MethodRefine} with different preferences.}
To evaluate our assumption on the stratum-based prioritized preference, we also compare it against its opposite prioritization where tasks are inserted in as abstract compound tasks as possible (`MR-M-Abs'), and the case with no preference where any completion profile is allowed (`MR-M-rand').
We choose the domains with the middle completeness.

From the experimental results (Figure \ref{fig:experiment2}), our stratum-based prioritized preference outperforms the other two cases.
It is because the completion profile associates the inserted tasks to a more abstract task and it generates a more case-specific method which may not suit other instances.
\section{Discussion and Conclusion}
Without declarative goals,
we suppose that in the original method set, every compound action at least has a method to decompose.
Our approach can also accept a declarative goal: we can trivially introduce a compound action of achieving the goal which is decomposed into a ``verifying'' action whose precondition is the goal and whose effect is empty.
%Note that we only invoke a TIHTN planner to obtain plans for refining methods and focus on HTN problems.
%Also, the TIHTN planner needs to search actions to insert from the vast number of action candidates, a refined method including the missing subtasks helps to find the plan.

To sum up, we present a framework to help domain experts to improve HTN domains by refining methods.
The experiment results demonstrate that our approach outperforms the method learning approach, HTN-MAKER, given an appropriately incomplete method set as input.
It is also illustrated that the stratum-based prioritized preference is effective to solve new instances.

In this paper, we assume primitive action models are available. In the future, it would be interesting to remove this assumption and study learning HTN methods and action models simultaneously from history data \cite{DBLP:journals/ai/Zhuo014,DBLP:journals/ai/AinetoCO19}.
Also, in this paper we only focus on keeping the procedural knowledge from domain experts and it would be a promising avenue to jointly consider partial annotations of tasks or state constraints together \cite{XiaoHPWS17htn,DBLP:conf/aaai/Munoz-AvilaDR19}.
%It would be also a promising avenue to study the feasibility of
In what follows, we are going to study refining methods directly from raw data, such as texts \cite{DBLP:conf/ijcai/FengZK18} %. %or images \cite{DBLP:conf/aaai/AsaiF18}.
and to explore plan recognition techniques with incomplete action models \cite{DBLP:journals/tist/Zhuo19} to improve HTN domains.

\bibliographystyle{aaai}
\bibliography{ref}

\end{document}